\documentclass[conference]{IEEEtran}
\pdfoutput=1
\IEEEoverridecommandlockouts
\makeatletter
\let\NAT@parse\undefined
\makeatother
\usepackage{hyperref}
\usepackage{cite}
\usepackage{lineno}
\usepackage{amsmath,amssymb,amsfonts}
\usepackage{graphicx}
\usepackage{textcomp}
\usepackage[table,xcdraw]{xcolor}
\usepackage{epstopdf}
\usepackage{amsthm} % load proof
\usepackage{multirow}
\usepackage{threeparttable}
\usepackage{ragged2e}
\usepackage{mathrsfs} % load \mathscr
\usepackage[ruled,linesnumbered]{algorithm2e}
\newtheorem{definition}{Definition}
\newtheorem{theorem}{Theorem}

\def\BibTeX{{\rm B\kern-.05em{\sc i\kern-.025em b}\kern-.08em
    T\kern-.1667em\lower.7ex\hbox{E}\kern-.125emX}}
\begin{document}

\title{Structure-Preference Enabled Graph Embedding Generation under Differential Privacy}

\author{
    \IEEEauthorblockN{Sen Zhang, Qingqing Ye, Haibo Hu$^{\ast}$\thanks{$^\ast$Corresponding author.}}
    \IEEEauthorblockA{The Hong Kong Polytechnic University, Hong Kong}
    \IEEEauthorblockA{\{senzhang, qqing.ye, haibo.hu\}@polyu.edu.hk}
}

%\author{\IEEEauthorblockN{1\textsuperscript{st} Given Name Surname}
%\IEEEauthorblockA{\textit{dept. name of organization (of Aff.)} \\
%\textit{name of organization (of Aff.)}\\
%City, Country \\
%email address or ORCID}
%\and
%\IEEEauthorblockN{2\textsuperscript{nd} Given Name Surname}
%\IEEEauthorblockA{\textit{dept. name of organization (of Aff.)} \\
%\textit{name of organization (of Aff.)}\\
%City, Country \\
%email address or ORCID}
%\and
%\IEEEauthorblockN{3\textsuperscript{rd} Given Name Surname}
%\IEEEauthorblockA{\textit{dept. name of organization (of Aff.)} \\
%\textit{name of organization (of Aff.)}\\
%City, Country \\
%email address or ORCID}
%\and
%\IEEEauthorblockN{4\textsuperscript{th} Given Name Surname}
%\IEEEauthorblockA{\textit{dept. name of organization (of Aff.)} \\
%\textit{name of organization (of Aff.)}\\
%City, Country \\
%email address or ORCID}
%\and
%\IEEEauthorblockN{5\textsuperscript{th} Given Name Surname}
%\IEEEauthorblockA{\textit{dept. name of organization (of Aff.)} \\
%\textit{name of organization (of Aff.)}\\
%City, Country \\
%email address or ORCID}
%\and
%\IEEEauthorblockN{6\textsuperscript{th} Given Name Surname}
%\IEEEauthorblockA{\textit{dept. name of organization (of Aff.)} \\
%\textit{name of organization (of Aff.)}\\
%City, Country \\
%email address or ORCID}
%}

\maketitle

\begin{abstract}
Graph embedding generation techniques aim to learn low-dimensional vectors for each node in a graph and have recently gained increasing research attention. Publishing low-dimensional node vectors enables various graph analysis tasks, such as structural equivalence and link prediction. Yet, improper publication opens a backdoor to malicious attackers, who can infer sensitive information of individuals from the low-dimensional node vectors. Existing methods tackle this issue by developing deep graph learning models with differential privacy (DP). However, they often suffer from large noise injections and cannot provide structural preferences consistent with mining objectives. Recently, skip-gram based graph embedding generation techniques are widely used due to their ability to extract customizable structures. Based on skip-gram, we present SE-PrivGEmb, a structure-preference enabled graph embedding generation under DP. For arbitrary structure preferences, we design a unified noise tolerance mechanism via perturbing non-zero vectors. This mechanism mitigates utility degradation caused by high sensitivity. By carefully designing negative sampling probabilities in skip-gram, we theoretically demonstrate that skip-gram can preserve arbitrary proximities, which quantify structural features in graphs. Extensive experiments show that our method outperforms existing state-of-the-art methods under structural equivalence and link prediction tasks.
\end{abstract}

% To mitigate utility degradation caused by high sensitivity, we propose a unified noise tolerance mechanism that perturbs non-zero vectors for arbitrary structure preferences.
% Since structural features can be quantified as \emph{node proximity}, we prove theoretically that skip-gram can preserve arbitrary proximity by carefully designing negative sampling probabilities. 
% This mechanism mitigates the utility degradation caused by high sensitivity. 
% we deeply analyze its optimization and propose a unified noise tolerance mechanism using a non-zero vector perturbation approach. This mechanism mitigates utility degradation caused by high sensitivity to structure preferences.
% for arbitrary structure-preference setting
% lack theoretical analysis of the relationships they preserve in the embedding space.
% This ensures that privacy amplification can be used to enhance privacy protection without compromising utility guarantees.


\section{Introduction}\label{Intro}
In recent years, graph embedding generation has gained significant research attention due to its ability to represent nodes with low-dimensional vectors while preserving the inherent properties and structure of the graph. The publication of low-dimensional node vectors facilitates a broad spectrum of graph analysis tasks, including structural equivalence and link prediction. However, improper publication creates opportunities for malicious attackers to potentially infer sensitive individual information. Therefore, it is crucial to integrate privacy-preserving techniques into graph embedding generation methods before making them publicly available. 

Differential privacy (DP) is a well-studied statistical privacy model known for its rigorous mathematical framework. 
One common technique for achieving differentially private training is the combination of noisy Stochastic Gradient Descent (SGD) and advanced composition theorems such as Moments Accountant (MA)~\cite{abadi2016deep}. 
This combination, known as DPSGD, has been widely studied in recent years for publishing low-dimensional node vectors, as the advanced composition theorems can effectively manage the problem of excessive splitting of the privacy budget during optimization.
For instance, Yang \emph{et al.}~\cite{yang2020secure} privately publish low-dimensional node vectors by extending the MA mechanism and applying it to graph models based on Generative Adversarial Networks (GANs) and Variational Autoencoders (VAEs). Another line of research~\cite{olatunji2021releasing, daigavane2021node, zhang2022towards, sajadmanesh2023gap, sajadmanesh2023progap, xiang2023preserving} explores graph embedding generation by applying DPSGD to Graph Neural Networks (GNNs). Despite their success, DPSGD is inherently suitable for structured data with well-defined individual gradients, while scaling poorly for graph learning models. The main reason is that individual examples in a graph are no longer independently computed for their gradients. As a result, \emph{existing methods typically suffer from large noise caused by high sensitivity}, making it difficult to strike a balance between privacy and utility. Additionally, \emph{none of existing methods can offer structure-preference settings}. Setting preferences is essential for extracting specific graph structures that align with mining objectives, improve predictive accuracy, and yield meaningful insights.

To address these issues, we turn to the skip-gram model~\cite{perozzi2014deepwalk, tang2015line, tang2015pte, grover2016node2vec}, an advanced graph embedding generation technique known for its capability for customizable structure extraction. 
Building on skip-gram, we propose SE-PrivGEmb, a novel approach for achieving differentially private graph embedding generation while supporting structure-preference settings. Achieving this target faces two main challenges: 

\begin{itemize}
    \item \emph{High Sensitivity.} Different structure-preference settings will result in varying sensitivities. The maximum sensitivity for arbitrary preference settings can be proportional to the size of the batch sampling, which leads to large noise injection and poor data utility.  
    \item \emph{Theoretical Analysis.} Previous skip-gram based methods~\cite{perozzi2014deepwalk, tang2015line, tang2015pte, grover2016node2vec, du2018dynamic} only preserve the structural feature related to node degree. They lack theoretical analysis of exactly what relationships they preserve in their embedding space, which results in weak interpretability for structural preferences in graph embedding generation.
\end{itemize}

Inspired by the one-hot encoding used in skip-gram, which ensures that the gradient (typically represented as a matrix) updates only partial vectors, this provides a new perspective for addressing high sensitivity. To tackle the first challenge, we conduct an in-depth analysis of skip-gram optimization, and design a unified noise tolerance mechanism by perturbing non-zero vectors to accommodate arbitrary structure preferences. This mechanism effectively mitigates the utility degradation caused by high sensitivity. 
To address the second challenge, we carefully design the negative sampling probability in skip-gram and theoretically prove that skip-gram can preserve arbitrary node proximities (which quantify structural features). Through formal privacy analysis, we demonstrate that the low-dimensional node vectors generated by SE-PrivGEmb satisfy node-level Rényi Differential Privacy (RDP).

The main contributions can be summarized as follows:
\begin{itemize}
% [leftmargin=5mm]
\item We present a novel method for achieving differentially private graph embedding generation, named SE-PrivGEmb. To our best knowledge, it is the \emph{first} method that supports structure-preference settings while satisfying node-level RDP and preserving high data utility.
\item We deeply analyze the optimization of skip-gram and design a unified noise tolerance mechanism via perturbing non-zero vectors to accommodate arbitrary structure preferences. This mechanism can effectively mitigate the utility degradation caused by high sensitivity.
\item We theoretically prove that the skip-gram can preserve arbitrary node proximities by carefully designing its negative sampling probabilities. 
This ensures that one can achieve structural preferences that align with the desired mining objectives.
\item Extensive experimental results on several large-scale networks (e.g., social networks, citation networks) demonstrate that our proposed method outperforms various state-of-the-art methods across structural equivalence and link prediction tasks.
\end{itemize}

The outline of the paper is as follows. Section~\ref{Preliminary} introduces the preliminaries. In Section~\ref{sec:ProDes}, we define the problem and give a first-cut solution. Section~\ref{Alg_Section} explains our proposed SE-PrivGEmb, and its privacy and time complexity are discussed in Section~\ref{subsec:priv_complex}. The experimental results are presented in Section~\ref{Exp_Section}, while related work is reviewed in Section~\ref{Related_work}. Finally, our research findings are summarized in Section~\ref{ConcluSection}.

\section{Preliminaries}\label{Preliminary}
In this section, we provide an overview of graph embedding generation, differential privacy, DPSGD, and node proximity, highlighting their key concepts and important properties. We also summarize the notations commonly used in this paper in Table~\ref{SymbolTable}.

\begin{table}[htb]
  \centering
  \begin{threeparttable}
      \caption{Frequently used symbols}\label{SymbolTable}
      \begin{tabular}{l|l}
        \hline   %  or \cline{col1-col2}
        Symbol & Description \\
        \hline   %  or \cline{col1-col2}
        {$\epsilon, \delta$} & {Privacy parameters} \\
        {$G, \mathbf{A}$} & {Original graph and its adjacency matrix} \\
        {$\mathcal{A}$} & {Randomized algorithm} \\
        {$D_\alpha$} & {Rényi divergence of order $\alpha$} \\
        {$|V|, |E|$}  & {Number of nodes and edges in $G$} \\
        {$\mathbf{x},\mathbf{y}$} & {Lowercase letters denoting vectors}  \\
        {$\mathbf{x}\cdot\mathbf{y}$} & {Inner product between two vectors}  \\
        {$\mathbf{X},\mathbf{Y}$} & {Bold capital letters denoting matrices} \\
        {$L$} & {Loss function} \\
        {$\mathbf{W}_{in},\mathbf{W}_{out}$} & {Embedding matrices of skip-gram} \\
        {$k$} & {Negative sampling number} \\
        {$r$} & {Dimension of low-dimensional vectors} \\
        $\mathbf{I}$ & Identity matrix \\
        \hline
      \end{tabular}
  \end{threeparttable}
\end{table}

\subsection{Graph Embedding Generation}\label{Original_SGM}
\begin{figure}[htb!]
  \centering
  \includegraphics[width=3in]{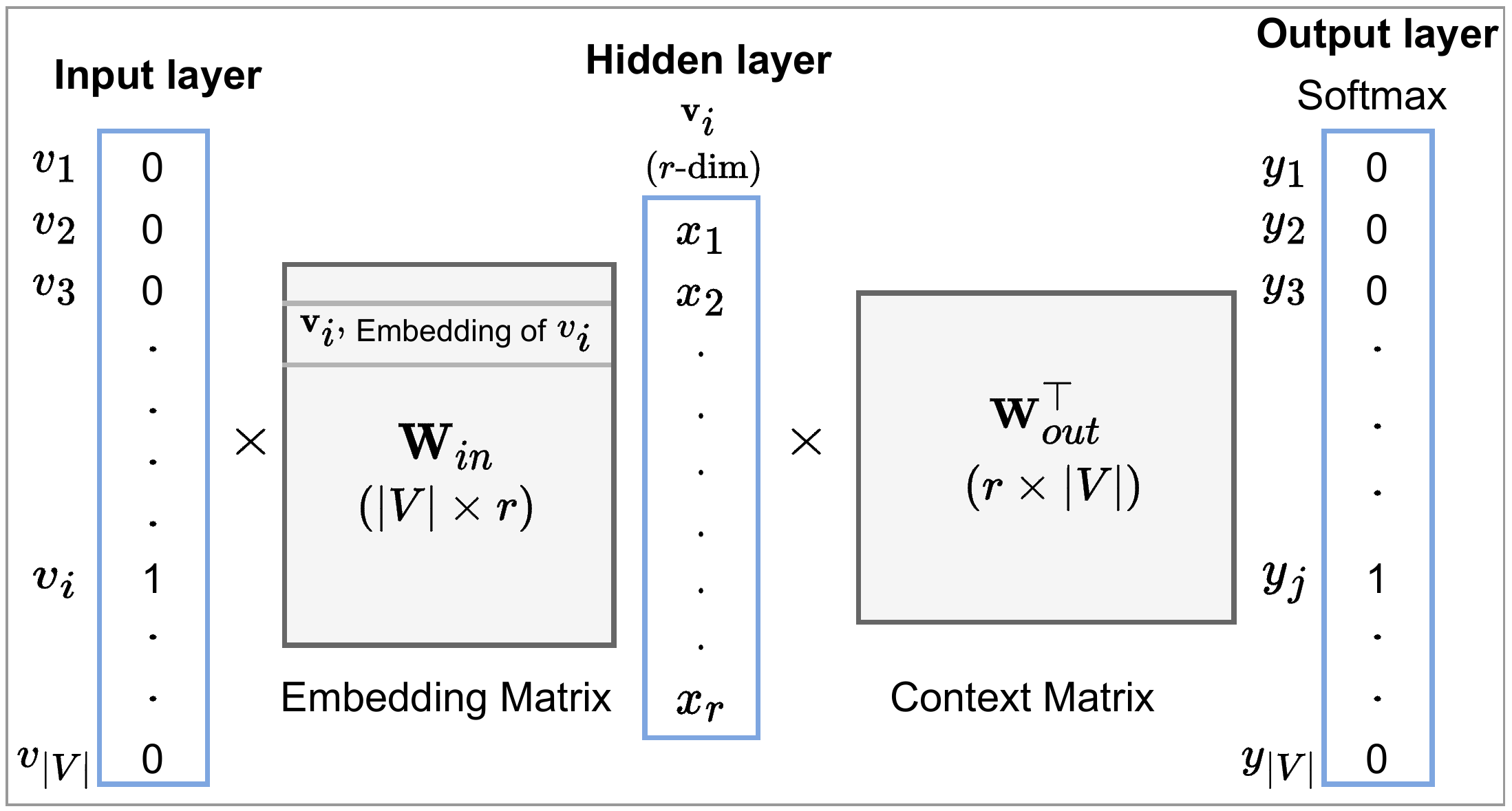}
  \caption{Architecture of a Skip-gram Model. The embedding matrices $\mathbf{W}_{in}$, with a size of $|V|\times{r}$, and $\mathbf{W}_{out}^\top$, with a size of $r\times|V|$, are two model parameters that require optimization. For any node pair $(v_i, v_j)$, $\mathbf{v}_i$ represents the vector for $v_i$ in the input weight matrix $\mathbf{W}_{in}$, while $\mathbf{v}_j$ represents the vector for $v_j$ in the output weight matrix $\mathbf{W}_{out}$.}
  \label{fig:SkipGram_ProFig}
\end{figure}

We consider an undirected and unweighted graph $G=(V,E)$, where $V$ is the set of nodes and $E$ is the set of edges.
For any two nodes $v_i$ and $v_j$ belonging to $V$, $(i, j)\in E$ represents an edge in $G$.
The adjacency matrix of $G$ is $\mathbf{A}$, where $\mathbf{A}_{ij}=1$ if $(v_i,v_j) \in E$ and $\mathbf{A}_{ij}=0$ otherwise.
The goal of graph embedding generation techniques is to learn a function $f: V\rightarrow{\mathbf{W}}$, where $\mathbf{W}\in\mathbb{R}^{|V|\times{r}}$ with embedding dimension $r\ll|V|$,
while preserving the inherent properties and structures of the original $G$. We denote $\mathbf{v}_i = f(v_i)$ as the vector embedding of the node $v_i$.

The skip-gram model is a neural network architecture that learns word embeddings by predicting the surrounding words given a target word. In the context of graph embedding, each node in the network can be thought of as a ``word'', and the surrounding words are defined as the nodes that co-occur with the target node in the network.
Inspired by this setting, DeepWalk~\cite{perozzi2014deepwalk} achieves graph embedding generation by treating the paths traversed by random walks over networks as sentences and using skip-gram (see Fig.\,\ref{fig:SkipGram_ProFig}) to learn latent representations of nodes. With the advent of DeepWalk, several skip-gram based graph embedding generation models have emerged,
including LINE~\cite{tang2015line}, PTE~\cite{tang2015pte}, and node2vec~\cite{grover2016node2vec}. These methods aim at learning informative node representations to predict neighboring nodes.
As a result, these methods share a common objective of maximizing the log-likelihood function, i.e. $\max p\left(v_j \mid v_i\right)$, and in machine learning, the convention is to minimize the cost function:
\begin{equation}\label{eq:naive_objFunc}
 \min -p\left(v_j \mid v_i\right)
 =-\frac{\exp(\mathbf{v}_i\cdot\mathbf{v}_j)}{\sum_{v^n \in V}\exp(\mathbf{v}_i\cdot\mathbf{v}^n)}.    
\end{equation}

However, the objective of Eq.\,(\ref{eq:naive_objFunc}) is hard to optimize, due to the costly summation over all inner product with every vertex of the graph. To improve the training efficiency, we adopt the skip-gram with negative sampling, which turns to optimize the following objective
function for each observed $(v_i,v_j)$ pair:
\begin{equation}\label{eq:line_obj}
  \begin{split}
    &\min L(v_i,v_j) \\ 
    =& -\log \sigma\left(\mathbf{v}_j \cdot \mathbf{v}_i\right) - \sum_{n=1}^k \mathbb{E}_{v^n \sim \mathbb{P}_{n}(v)}\left[\log \sigma\left(-\mathbf{v}^n \cdot \mathbf{v}_i\right)\right],
  \end{split}
\end{equation}
where $k$ is the number of negative samples, $\mathbf{v}_i\in\mathbf{W}_{in}$ is the central vector of $v_i$, $\mathbf{v}_j\in\mathbf{W}_{out}$ is the context vector of $v_j$,
$\sigma(\mathbf{v}_i\cdot\mathbf{v}_j)$ is the normalized similarity of $v_i$ and $v_j$ under the their representations $\mathbf{v}_i$ and $\mathbf{v}_j$ learned from the model, $\sigma(\cdot)=\frac{1}{1+\exp(-x)}$ is the classic \emph{Sigmoid} activation function. 
$\mathbb{P}_{n}(v)$ is the noise distribution for negative sampling, which is carefully designed in Section~\ref{sec:opt_emb} of this paper to achieve provable structure preference.

\subsection{Differential Privacy}\label{pre:diff_priv}
\textbf{$(\epsilon,\delta)$-DP.}
Differential privacy (DP)~\cite{dwork2006calibrating} has emerged as the de-facto standard
notion in private data analysis. In general, it requires an algorithm to be insensitive to the changes in any individual's record. In the context of graph data, two graph datasets $G$ and $G^\prime$ are considered as neighboring databases if they differ by one record (i.e., edge or node). Formally, DP for graph data is defined as follows.  
\begin{definition}[Edge (Node)-Level DP~\cite{hay2009accurate}]\label{DP_def}
A graph analysis mechanism $\mathcal{A}$ achieves edge (node)-level $(\epsilon, \delta)$-DP, if for any pair of input graphs $G$ and $G^\prime$ that are neighbors (differ by at most one edge or node), and for all possible $O \subseteq Range(\mathcal{A})$, we have: 
$\mathbb{P}[\mathcal{A}(G) \in O] \leq \exp(\epsilon) \cdot \mathbb{P}[\mathcal{A}(G^\prime) \in O] + \delta$.
\end{definition}

The concept of the neighboring dataset $G$, $G^\prime$ is categorized into two types. Specifically, if $G^\prime$ can be derived by replacing a single data instance in $G$, it is termed bounded DP~\cite{dwork2006calibrating}. If $G^\prime$ can be obtained by adding or removing a data sample from $G$, it is termed unbounded DP~\cite{dwork2006differential}. The parameter $\epsilon$ represents the privacy budget, which determines the trade-off between privacy and utility in the algorithm. A smaller value of $\epsilon$ indicates a higher level of privacy protection. The parameter $\delta$ is typically chosen to be very small and is informally referred to as the failure probability.

\textbf{$(\alpha,\epsilon)$-RDP.} In this paper, we use an alternative definition of DP, called RDP~\cite{mironov2017renyi}, which allows obtaining tighter sequential composition results:

\begin{definition}[RDP~\cite{mironov2017renyi}]
Given $\alpha>1$ and $\epsilon>0$, a randomized algorithm $\mathcal{A}$ satisfies $(\alpha, \epsilon)$-RDP if for every neighboring datasets $G$ and $G^{\prime}$, we have 
$D_\alpha\left(\mathcal{A}(G) \| \mathcal{A}\left(G^{\prime}\right)\right) \leq \epsilon$,
where $D_\alpha(P \| Q)$ is the Rényi divergence of order $\alpha$ between probability distributions $P$ and $Q$ defined as
$D_\alpha(P \| Q)=\frac{1}{\alpha-1} \log \mathbb{E}_{x \sim Q}\left[\frac{P(x)}{Q(x)}\right]^\alpha$.
\end{definition}

A crucial feature of RDP is that it can be transformed into standard $(\epsilon, \delta)$-DP using Proposition 3 from~\cite{mironov2017renyi}. This conversion can be done in the following way:

\begin{theorem}[RDP conversion to ($\epsilon, \delta)$-DP~\cite{mironov2017renyi}]\label{Theo:RDP_to_DP}
If $\mathcal{A}$ is an $(\alpha, \epsilon)$-RDP algorithm, then it also satisfies $(\epsilon+\frac{\log (1 / \delta)}{\alpha-1}, \delta)$-DP for any $\delta \in(0,1)$.
\end{theorem}

\textbf{Gaussian mechanism.}
Suppose a function $f$ maps a graph $G$ to an $r$-dimensional output in $\mathbb{R}^r$. To create a differentially private mechanism of $f$, it is common practice to inject random noise into the output of $f$. The size of this noise is determined by the sensitivity of $f$, which is defined as follows.
\begin{definition}[Sensitivity~\cite{dwork2006calibrating}]\label{SensDef}
Given a function $f:G\rightarrow\mathbb{R}^r$, for any neighboring datasets $G$ and $G^\prime$, the $\ell_2$-sensitivity of $f$ is defined as $S_{f} = \max_{G, G^\prime}\|f(G) - f(G^\prime)\|_2$.
\end{definition}

The Gaussian mechanism is a commonly used approach to achieve RDP, whereby Gaussian noise is added to the output of an algorithm to protect user privacy. By adding Gaussian noise with variance $\sigma^2$ to function $f$, such that $\mathcal{A}(G) = f(G) + \mathcal{N}(\sigma^2 \mathbf{I})$, we can obtain an $(\alpha, \epsilon)$-RDP algorithm for all $\alpha > 1$, where $\epsilon=\frac{\alpha S_f^2}{2 \sigma^2}$~\cite{mironov2017renyi}.

It is worth noting that the concept of sensitivity implies that satisfying node-level DP is challenging as varying one node could result in the removal of $|V|-1$ edges in the worst case. As a result, a significant amount of noise must be added to ensure privacy protection. 

\textbf{Important properties.} RDP has two crucial properties that are essential when designing intricate algorithms from simpler ones:

\begin{itemize}
    \item \emph{Sequential Composition:} When an $(\alpha, \epsilon)$-RDP algorithm is applied multiple times ($m$ instances) on the same dataset, the resulting output will be at most $(\alpha, m\epsilon)$-RDP.
    \item \emph{Robustness to Post-processing:} The $(\alpha, \epsilon)$-RDP property remains intact even after any post-processing steps are applied to the output of an $(\alpha, \epsilon)$-RDP algorithm.
\end{itemize}

\subsection{DPSGD}\label{sec:dpsgd}
DPSGD~\cite{abadi2016deep} is an advanced training algorithm for deep learning models. Its main purpose is to manage the issue of excessive splitting of the privacy budget during optimization. In DPSGD, the gradient $\mathbf{g}\left(x_i\right)$ is computed for each example $x_i$ in a randomly selected batch of size $B$. The $\ell_2$ norm of each gradient is then clipped using a threshold $C$ to control the sensitivity of $\mathbf{g}\left(x_i\right)$. After clipping, the gradients are summed, and Gaussian noise $\mathcal{N}(C^2 \sigma^2\mathbf{I})$ is added to ensure privacy. The resulting noisy cumulative gradient $\tilde{\mathbf{g}}$ is used to update the model parameters by taking its average. The expression for $\tilde{\mathbf{g}}$ is as follows:
\begin{equation}\label{dpsgd_eq1}
  \tilde{\mathbf{g}} \leftarrow \frac{1}{B}\big(\sum_{i=1}^{B} Clip(\mathbf{g}\left(x_i\right))+\mathcal{N}(C^2 \sigma^2 \mathbf{I})\big),
\end{equation}
in which $Clip(\cdot)$ represents the clipping function defined as $Clip(\mathbf{g}\left(x_i\right))=\mathbf{g}\left(x_i\right)/\max(1,\frac{\|\mathbf{g}\left(x_i\right)\|_2}{C})$.

\subsection{Node Proximity}\label{sec:proximity}
Node proximity measures are commonly used to quantify the closeness or relationships between two or more nodes in a graph. A general definition of node proximity can be expressed as follows:
\begin{definition}[Node Proximity Matrix]
The node proximity matrix $\mathbf{P}$ of a graph $G$ is a $|V| \times |V|$ matrix that quantifies the closeness between pairs of nodes. Each element $p_{ij}\in\mathbf{P}$ reflects the proximity of nodes $v_i$ and $v_j$ based on their structural relationships in $G$. This proximity is defined as: 
\begin{equation}\label{eq:node_proximity}
    p_{ij} = g(N(v_i), N(v_j), G),     
\end{equation}
where $N(v_i)$ and $N(v_j)$ are the neighbors of nodes $v_i$ and $v_j$, and $g(\cdot)$ is a function that applies various measures such as common neighbors to determine their proximity. 
\end{definition}

The function $g(\cdot)$ can consist of first-order, second-order, or high-order features. First-order features, such as common neighbors and preferential attachment~\cite{barabasi1999emergence}, consider only the one-hop neighbors of the target nodes. Second-order features, like Adamic-Adar and resource allocation~\cite{zhou2009predicting}, are based on two-hop neighborhoods. High-order heuristics, which require knowledge of the entire network, include Katz~\cite{katz1953new}, PageRank~\cite{haveliwala2002topic}, and random walk-based proximity~\cite{yang2015network}.

\section{Problem Description and A First-cut Solution}
\subsection{Problem Description}\label{sec:ProDes}
\textbf{Threat Model.} In this paper, we consider a white-box attack~\cite{he2019model}, where the adversary possesses full information about the skip-gram model used for graph embedding, including its architecture and parameters. This means that the attacker has access to the published model but not the training data and training process. With the knowledge of some (or even all except the target) samples in the training dataset, the attacker aims to infer the existence of a target training data sample.

\textbf{Privacy Model.} 
In reality, data owners who have the ability to analyze graph structural features usually hope to be able to customize the extracted graph structural information. They aim to input the extracted features into a unified differentially private graph embedding generation model, with the goal of obtaining private low-dimensional node vectors that effectively support specific mining tasks. As stated in Section~\ref{sec:proximity}, the node proximity can quantify structural features. \emph{Given a graph $G=(V, E)$ and an arbitrary node proximity $p_{ij}$, we aim to achieve differentially private graph embedding generation while preserving arbitrary structure preferences.}
Formally, the definition of differentially private graph embedding generation is as follows: 

\begin{definition}[Private Graph Embedding Generation under Bounded DP]\label{def:priv_sg}
Let $\Theta$ be a collection of the input and output weight matrix $\{\mathbf{W}_{in}, \mathbf{W}_{out}\}$, that is $\Theta=\{\mathbf{W}_{in}, \mathbf{W}_{out}\}$. A graph embedding generation model $\mathcal{L}=\mathbb{E}_{(i,j)\in{E}}L(v_i,v_j)$ satisfies $(\epsilon, \delta)$-node-level DP under bounded constraints if two neighboring graphs $G$ and $G^\prime$, which differ in only a node, and for all
possible $\Theta_S\subseteq{Range(\mathcal{L})}$:
\begin{equation*}
  \mathbb{P}(\mathcal{L}(G) \in \Theta_S) \leq \exp (\epsilon) \times \mathbb{P}\left(\mathcal{L}\left(G^\prime\right) \in \Theta_S\right) + \delta.
\end{equation*}
where $\Theta_S$ denotes the set comprising all possible values of $\Theta$.
\end{definition}

As stated in~\cite{xiang2023preserving}, the definition of node-level private graph embedding generation is only applicable under bounded DP. 
Consider a pair of neighboring graphs $G^\ast\subseteq{G}$ and $G^\prime\subseteq{G}$, where $G^{\prime}=G^\ast\cup\{node \ {z} \}$. If we use unbounded DP as the privacy definition, we face a challenge when attempting to express the following inequality:
\begin{equation*}
    \mathbb{P}\left(\mathcal{L}_z\left(G^{\prime}\right) \in \Theta_S\right) \leq \exp(\epsilon) \times \mathbb{P}\left(\mathcal{L}_z\left(G^\ast\right) \in \Theta_S\right)+\delta.    
\end{equation*}

The issue arises because the node $z$ is not included in $G^\ast$. Therefore, the node-level differentially private graph embedding generation is incompatible with the definition of unbounded DP.

Based on the robustness to post-processing, we can immediately conclude that the $(\epsilon, \delta)$-private graph embedding generation is robust to graph downstream tasks, as stated in the following theorem:
\begin{theorem}
Let $\mathcal{L}$ be an $(\epsilon,\delta)$-node-level private graph embedding generation model, and $f$ be any graph downstream task whose input is the private graph embedding matrix (i.e., $\mathbf{W}_{in}$). Then, $f \circ {\mathcal{L}}$ satisfies $(\epsilon,\delta)$-node-level DP.
\end{theorem}

\subsection{A First-cut Solution}\label{sec:naive_method}
To establish structural preference settings in graph embedding generation, we design a novel objective function $L_{nov}$ for each edge by redefining the objective function $L$ in Eq.\,(\ref{eq:line_obj}), as follows: 
\begin{align}
     &\min L_{nov}(v_i,v_j,p_{ij}) = p_{ij}L(v_i,v_j), \label{eq:novel_edgeObj} \\
    =& -p_{ij}\log \sigma\left(\mathbf{v}_j \cdot \mathbf{v}_i\right) - p_{ij}\sum_{n=1}^k \mathbb{E}_{v^n \sim \mathbb{P}_{n}(v)}\left[\log \sigma\left(-\mathbf{v}^n \cdot \mathbf{v}_i\right)\right], \nonumber
\end{align}
where $p_{ij}$ denotes an \emph{arbitrary node proximity} that quantifies structural features.

As stated in Section~\ref{sec:dpsgd}, DPSGD with the advanced composition mechanism can manage the issue of excessive splitting of the privacy budget during optimization.
One straightforward approach is to perturb the gradient on $\mathbf{v}$, where \emph{$\mathbf{v}$ is a general notation representing either $\mathbf{v}_i$ or $\mathbf{v}_j$}. Formally, the noisy gradient $\widetilde{\nabla}_{\mathbf{v}}L_{nov}^B$ is expressed as follows:
\begin{equation}\label{eq:novNoiseGra_on_vi}
\begin{split}
    &\widetilde{\nabla}_{\mathbf{v}}L_{nov}^B(v_{i_b},v_{j_b},p_{i_bj_b})\\
    =&\frac{1}{B}\big(\sum_{b=1}^B Clip\big(\frac{\partial{L_{nov}(v_{i_b},v_{j_b},p_{i_bj_b})}}{\partial\mathbf{v}}\big)+\mathcal{N}(S_{\nabla_{\mathbf{v}}}^2\sigma^2\mathbf{I})\big),      
\end{split}
\end{equation}
where $B$ denotes the size of batch sampling, and the upper bound of $S_{\nabla_{\mathbf{v}_i}}$ is $BC$ under node-level DP, with $C$ being the clipping threshold. The main reason for significant noise is that in graph learning, individual examples no longer compute gradients independently.  

\textbf{Limitation.}
The approach discussed above, however, produces poor results. The root cause of the poor performance is the large sensitivity.
\section{Our Proposal: SE-PrivGEmb}\label{Alg_Section}
To overcome the above limitation, we present SE-PrivGEmb, a novel method for achieving differentially private graph embedding generation, which incorporates structural preferences while satisfying node-level RDP and high data utility.
In particular, we deeply analyze the optimization of skip-gram and design a noise tolerance mechanism via perturbing non-zero vectors (see Section~\ref{sec:nonZero_Perturb}), which can effectively mitigate the utility degradation caused by high sensitivity. We theoretically prove that the skip-gram can preserve arbitrary node proximities by carefully designing its negative sampling probabilities (see Section~\ref{sec:opt_emb}). This ensures that one can achieve structural preferences that align with the desired mining objectives.
The training algorithm is shown in Section~\ref{sec:train_alg}. 

\subsection{Noise Tolerance via Perturbing Non-zero Vectors}\label{sec:nonZero_Perturb}
For the gradient with respect to input weight matrix, that is $\frac{\partial L_{nov}}{\partial\mathbf{W}_{in}}$, it is equivalent to taking the derivative for the hidden layer (see Fig.\,\ref{fig:SkipGram_ProFig}). Specifically, 
 \begin{align}
        \frac{\partial L_{nov}}{\partial\mathbf{W}_{in}}
        =\frac{\partial L_{nov}}{\partial\mathbf{v}_i}
        =&p_{ij}\sum_{n=0}^k \left(\sigma\left(\mathbf{v}^n \cdot \mathbf{v}_i\right) - \mathbb{I}_{v_j}[v^n] \right) \cdot \mathbf{v}^n, \label{eq:sgm_vi_gradient}
\end{align}
since the input layer is a one-hot encoded vector. Here, $\mathbb{I}_{v_j}[v^n]$ is an indicator function to indicate whether $v^n$ is a positive node and when $n=0$, $v^n=v_j$.
 
For the gradient with respect to output weight matrix, that is 
 \begin{align}
          \frac{\partial L_{nov}}{\partial\mathbf{W}_{out}}
        = \frac{\partial L_{nov}}{\partial\mathbf{v}^n}
        =p_{ij}\left(\sigma\left(\mathbf{v}^n \cdot \mathbf{v}_i\right) - \mathbb{I}_{v_j}[v^n] \right) \cdot \mathbf{v}_i, \label{eq:sgm_vj_gradient}
\end{align}
with negative sampling, only a fraction of the node vectors in $\mathbf{W}_{out}$ are updated. Gradients for $k+1$ node vectors, representing both positive and negative nodes in $\mathbf{W}_{out}$, are computed using Eq.\,(\ref{eq:sgm_vj_gradient}). 

This encourages us to inject noise exclusively into non-zero vectors within the gradients of $\mathbf{W}_{in}$ and $\mathbf{W}_{out}$ to improve utility. As outlined in Section~\ref{sec:naive_method}, utilizing $\mathbf{v}$ as a general notation representing either $\mathbf{v}_i$ or $\mathbf{v}_j$, we have:
\begin{equation}\label{eq:nonZeroNoiseGra_on_vi}
    \begin{split}
     &\widetilde{\nabla}_{\mathbf{v}}L_{nov}^B(v_{i_b},v_{j_b},p_{i_bj_b}) \\
    =&\frac{1}{B}\big(\sum_{b=1}^B Clip\big(\frac{\partial{L_{nov}(v_{i_b},v_{j_b},p_{i_bj_b})}}{\partial\mathbf{v}}\big)+\widetilde{\mathcal{N}}(S_{\nabla_{\mathbf{v}}}^2\sigma^2\mathbf{I})\big),
    \end{split}
\end{equation}
where \emph{$\widetilde{\mathcal{N}}(\cdot)$ denotes the noise matrix that selectively adds noise to non-zero vectors in the gradient}. Taking $\mathbf{W}_{in}$ as an example, the resulting $\mathbf{W}_{in}$ using Eq.\,(\ref{eq:nonZeroNoiseGra_on_vi}) is illustrated in Fig.\,\ref{fig:updateInWeight}(d) after the gradient update.  

\begin{figure}[!htb]
%  \centering
  \centerline{\includegraphics[width = 3.5in]{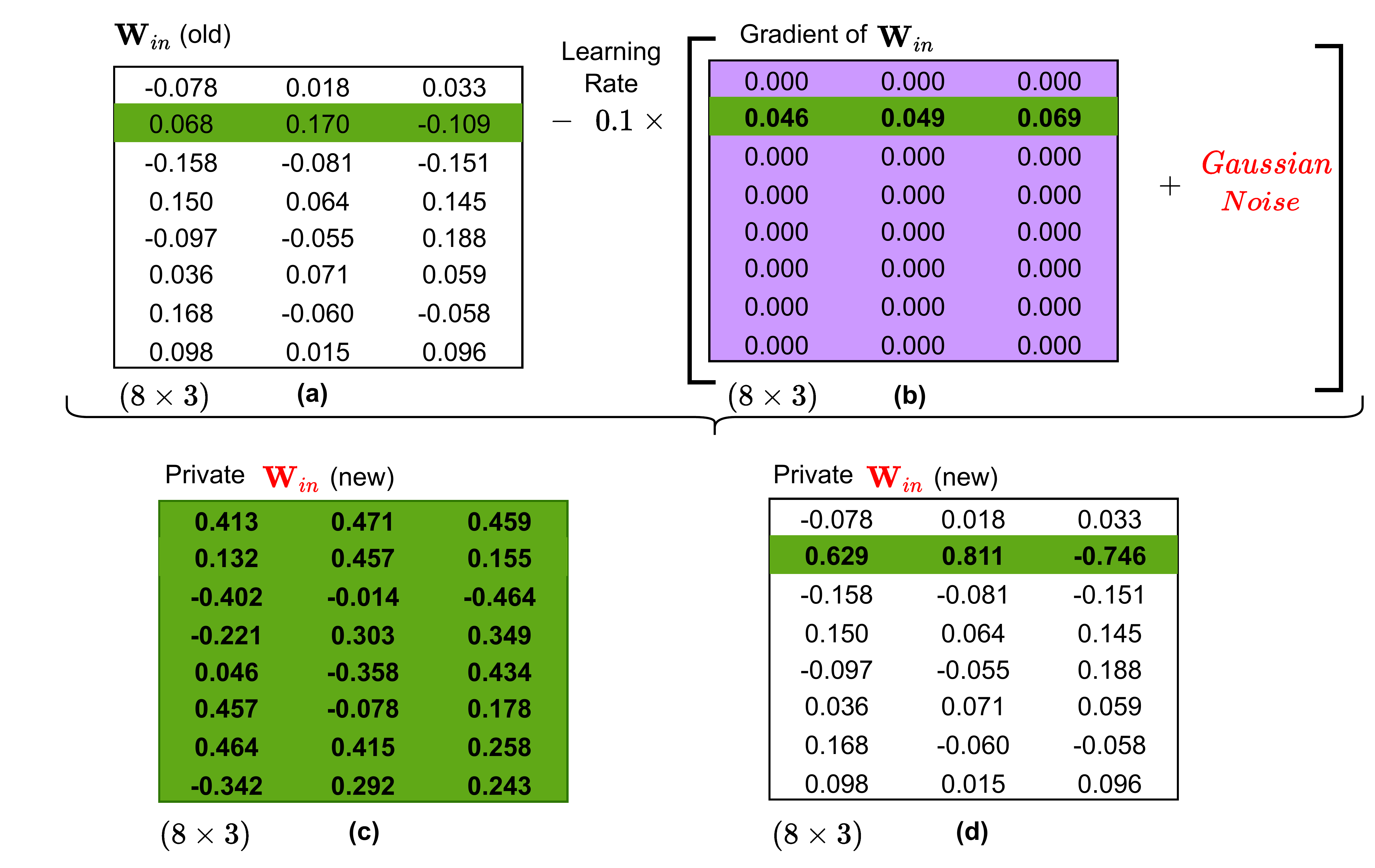}}
  \caption{An Illustration of Private Update for $\mathbf{W}_{in}$. (a) represents the original input weight, (b) represents the gradient of the input weight, while (c) and (d) depict the perturbed input weight using Eq.\,(\ref{eq:novNoiseGra_on_vi}) and Eq.\,(\ref{eq:nonZeroNoiseGra_on_vi}), respectively.}
  \label{fig:updateInWeight}
\end{figure}

It is worth mentioning, here, that as stated in Definition~\ref{def:priv_sg}, we only publish the final graph embedding matrices that satisfy DP after iterative optimization. Therefore, in this scenario, the gradient updates of each iteration in Fig.\,\ref{fig:updateInWeight}(b) are not published, so the attackers do not reveal information about which edges are sampled. 

\subsection{Theoretical Guarantee on Structure Preference}\label{sec:opt_emb}
In this section, by carefully designing the negative sampling probability $\mathbb{P}_n(v)$, we theoretically prove that the skip-gram objective can preserve arbitrary node proximities that quantify structural features. 

\begin{theorem}\label{theo:opti_emb}
Consider a node proximity matrix $\mathbf{P}$ for a given graph $G$, with its elements denoted as $p_{ij} \in \mathbf{P}$.
Let the negative sampling probability $\mathbb{P}_n(v)$ in objective~(\ref{eq:novel_edgeObj})
be proportional to $\frac{\min(\mathbf{P})}{\sum_{v_{j} \in {V}} p_{ij}}$,
where $\sum_{v_j \in {V}} p_{ij}$ represents the sum of proximity values associated with node $v_i$ across all the nodes in the graph, and 
% $\min(\mathbf{P}) = \min\{p_{ij} | p_{ij}>0, \forall i,j\}$.
% \begin{equation*}
$\min(\mathbf{P}) = \min\{p_{ij} | p_{ij}>0, \forall i,j\}$.
% \end{equation*}
Let $x_{ij} = \mathbf{v}_i \cdot \mathbf{v}_j$ in Eq.\,(\ref{eq:novel_edgeObj}). We can derive the theoretical optimal solution in Eq.\,(\ref{eq:novel_edgeObj}) with respect to $x_{ij}$ as follows:

\begin{equation}\label{eq:opt_emb}
    x_{ij} = \mathbf{v}_i \cdot \mathbf{v}_j
    = \log \big(\frac{p_{ij}}{k\min(\mathbf{P})}\big),
\end{equation}
where $k$ and $\min(\mathbf{P})$\footnote{$\min(\mathbf{P})$ is a constant for a given graph dataset, which is derived from precomputing its proximity matrix (see Line~\ref{code:comp_prox} in Algorithm~\ref{alg:PrivSkipGram}). $\min(\mathbf{P})$ can be substituted with another constant, such as $\mathcal{O}(\min(\mathbf{P}))$, provided that the probability $\frac{\mathcal{O}(\min(\mathbf{P}))}{\sum_{v_{j} \in {V}} p_{ij}}\in (0,1)$.} are both constants.
According to Eq.\,(\ref{eq:opt_emb}), we observe that $x_{ij}$ preserves the logarithmic value of $p_{ij}$ with a constant shift.
\end{theorem}

\begin{proof}
For ease of understanding, we first show the expected objective function of skip-gram as follows:
    \begin{equation}\label{eq:expected_obj}
        \begin{split}
            \min_{\mathbf{v}_i, \mathbf{v}_j} \mathcal{L}
            =&\mathbb{E}_{(i,j) \in E} L_{nov}(v_i,v_j,p_{ij}).
        \end{split}
    \end{equation}
    
Consider an arbitrary proximity matrix $\mathbf{P}$ of a graph, where $p_{ij}\in\mathbf{P}$.
We rewrite the objective~(\ref{eq:expected_obj}) as follows:

    \begin{equation*}
        \begin{split}
             \min_{\mathbf{v}_i, \mathbf{v}_j}\mathcal{L}
            =&\sum_{v_i\in {V}} \sum_{v_j\in {V}} p_{ij} \cdot\Bigl(-\log \sigma\left(\mathbf{v}_{j} \cdot \mathbf{v}_{i}\right)\\
            &- k \cdot \mathbb{E}_{v^n\sim{\mathbb{P}_n(v)}} \log \sigma\left(-\mathbf{v}_n \cdot \mathbf{v}_{i}\right)\Bigr)\\
            =&-\sum_{v_i\in {V}} \sum_{v_j\in {V}} p_{ij} \cdot\log \sigma\left(\mathbf{v}_{j} \cdot \mathbf{v}_{i}\right)\\
            & - \sum_{v_i\in {V}} \sum_{v_j\in {V}} p_{ij} \cdot \left(k \cdot \mathbb{E}_{v^n\sim{\mathbb{P}_n(v)}} \log \sigma\left(-\mathbf{v}_n \cdot \mathbf{v}_{i}\right)\right)\\
            =&-\sum_{v_{i} \in {V}} \sum_{v_{j} \in {V}}
            p_{ij} \cdot \log \sigma\left(\mathbf{v}_{j} \cdot \mathbf{v}_{i}\right) - \\
            &\sum_{v_{i} \in {V}} \big(\sum_{v_{j} \in {V}} p_{ij}\big) \cdot k \cdot \mathbb{E}_{v^n\sim{\mathbb{P}_n(v)}} \log \sigma\left(-\mathbf{v}_n \cdot \mathbf{v}_i\right).
        \end{split}
    \end{equation*}

        In this context, we define $\mathbb{P}_n(v)$ to be proportional to $\frac{\min(\mathbf{P})}{\sum_{v_{j} \in {V}} p_{ij}}$, and can explicitly express the expectation term as:
        \begin{flalign}\label{eq:new_negSmp}
            &\mathbb{E}_{v^n\sim{\mathbb{P}_n(v)}}\log \sigma\left(-\mathbf{v}_n \cdot \mathbf{v}_{i}\right) \\
            =&\sum_{v^n \in V} \frac{\min(\mathbf{P})}{\sum_{v_{j} \in V} p_{ij}}\log \sigma\left(-\mathbf{v}_n \cdot \mathbf{v}_{i}\right). \notag
        \end{flalign}

        With Eq.\,(\ref{eq:new_negSmp}), we have:
        \begin{equation}\label{eq:simplified_objFunc}
            \begin{split}
                \min_{\mathbf{v}_i,\mathbf{v}_j}\mathcal{L} = & -\sum_{v_{i} \in {V}} \sum_{v_{j} \in {V}}
                p_{ij} \cdot \log \sigma\left(\mathbf{v}_{j} \cdot \mathbf{v}_{i}\right) \\
                & - \sum_{v_{i} \in V} \sum_{v_{j} \in V} k \cdot \min(\mathbf{P}) \cdot \log \sigma\left(-\mathbf{v}_{j} \cdot \mathbf{v}_{i}\right).
            \end{split}
        \end{equation}

        To minimize Eq.\,(\ref{eq:simplified_objFunc}), we set the partial derivative of each independent variable $x_{ij} = \mathbf{v}_{j} \cdot \mathbf{v}_{i}$ to zero:

        \begin{equation*}
            \begin{split}
                 \frac{\partial {\mathcal{L}}}{\partial x_{ij}}
                =&\sum_{v_{i} \in {V}} \sum_{v_{j} \in {V}}\left((p_{ij} + k\cdot{\min(\mathbf{P})})\cdot\sigma(x_{ij}) - p_{ij}\right) = 0.
            \end{split}
        \end{equation*}

        This leads to the following expression:
        \begin{equation*}\label{Best_results}
            x_{ij} = \log \big(\frac{p_{ij}}{k\min(\mathbf{P})}\big),
        \end{equation*}
        which completes the proof.
\end{proof}

In Theorem~\ref{theo:opti_emb}, $x_{ij}$ represents the inner product of $v_i$ and $v_j$ to estimate the proximity of node pairs $(v_i, v_j)$. From the optimal solution in Eq.\,(\ref{eq:opt_emb}), we observe that $x_{ij}$ preserves the logarithmic value of $p_{ij}$ with a constant shift, indicating that our method maintains the proximity between any two nodes.

\textbf{Comparison with Prior Works.} In what follows, we also compare our findings with the optimal embedding presented in prior research to better reflect our contribution.

According to previous research~\cite{qiu2018network, du2018dynamic}, the expectation term within the objective~(\ref{eq:expected_obj}) can be expressed as:

\begin{equation}\label{eq:neg_exp_term}
\begin{split}
     \mathbb{E}_{v^n \sim \mathbb{P}_{n}(v)}\left[\log \sigma\left(-\mathbf{v}^n \cdot \mathbf{v}_i\right)\right] 
    =\sum_{v^n \in V}\frac{d_{v^n}}{\mathcal{D}}\log \sigma\left(-\mathbf{v}^n \cdot \mathbf{v}_i\right),   
\end{split}
\end{equation}
where $\mathbb{P}_{n}(v) \propto \frac{d_{v}^n}{|E|}$ is used for negative sampling, $d_v$ represents the degree of node $v$, and $\mathcal{D}=\sum_{ij}p_{ij}$. By fusing Eq.\,(\ref{eq:neg_exp_term}), the theoretical optimal solution of the objective~(\ref{eq:expected_obj}) is:
\begin{equation}\label{eq:pre_optimalSolution}
  \mathbf{v}_j \cdot \mathbf{v}_i = \log \big(\frac{p_{ij}\mathcal{D}}{d_id_j}\big) - \log \big(k\big).
\end{equation}

The optimal embedding in Eq.\,(\ref{eq:pre_optimalSolution}) fails to preserve the exact relationships for specified structural information within the embedding space.

\subsection{Training Algorithm}\label{sec:train_alg}
The overall training algorithm of SE-PrivGEmb is shown in Algorithm~\ref{alg:PrivSkipGram}. 
We compute the node proximity matrix for a graph $G$ (line~\ref{code:comp_prox}), and divide this graph into sets of disjoint subgraphs $G_\mathcal{S} = \{\mathcal{S}_1, \mathcal{S}_2, \cdots, \mathcal{S}_{|E|}\}$ (line~\ref{code:graph_subset}), where each $\mathcal{S}_i$ obtained through Algorithm~\ref{Alg:edge_samp} consists of an edge and its corresponding negative samples\footnote{To facilitate privacy analysis in Section~\ref{sec:priv_analy}, we collect positive and negative samples prior to training. An alternative solution is to implement Algorithm~\ref{Alg:edge_samp} during the training of Algorithm~\ref{alg:PrivSkipGram}. However, the sampling probability needed for privacy amplification involves more complex analysis, which we will address in future work.}. 
We initialize the weight matrices $\mathbf{W}_{in}$ and $\mathbf{W}_{out}$ (line~\ref{code:weight_init}). In each epoch, we sample $B$ subgraphs uniformly at random according to Algorithm~\ref{Alg:edge_samp} (line~\ref{code:samp_edges}). Next, it updates $\mathbf{W}_{in}$ according to $\widetilde{\nabla}_{\mathbf{v}_i}L_{nov}^B$ (line~\ref{code:update_w_in}), and updates $\mathbf{W}_{out}$ according to $\widetilde{\nabla}_{\mathbf{v}_j}L_{nov}^B$ (line~\ref{code:update_w_out}). Throughout the training process, the algorithm also updates the overall privacy cost $\epsilon$ to ensure it does not exceed the target privacy budget (lines~\ref{code:cal_RDP}-\ref{code:stop_opt}). Finally, after completing all epochs, the algorithm returns differentially private $\mathbf{W}_{in}$ and $\mathbf{W}_{out}$.

\begin{algorithm}
\caption{Generating Disjoint Subgraphs}\label{Alg:edge_samp}
\KwIn{graph $G$.}
\KwOut{Set of subgraphs, $G_\mathcal{S}$.}
Set $G_\mathcal{S}=\{\}$ and $\mathcal{S}=\{\}$\;
\For{$(i,j) \in E$}
{
    Assign $(v_i,v_j)$ to $\mathcal{S}$\;
    \For{$n = 1, 2, \cdots, k$}
    {\label{code:NegSam_start}
        \While{\text{True}}
        {
        $v_n\leftarrow$ randomly sample one node from $V$\;
        \If{$(v_i,v_n)$ is not in $E$}
        {
            Break\;
        }
        }
        Assign $(v_i,v_n)$ to $\mathcal{S}$\;
    }\label{code:NegSam_end}
    Assign $\mathcal{S}$ to $G_\mathcal{S}$\;
}
\Return $G_\mathcal{S}$\;
\end{algorithm}

\begin{algorithm}
\caption{SE-PrivGEmb Algorithm}\label{alg:PrivSkipGram}
\KwIn{
privacy parameters $\epsilon$ and $\delta$,
standard deviation $\sigma$,
learning rate $\eta$,
embedding dimension $r$,
negative sampling number $k$,
batch size $B$,
gradient clipping threshold $C$,
number of maximum training epochs $n^{epoch}$.}
\KwOut{Differentially Private $\mathbf{W}_{in}$, $\mathbf{W}_{out}$.}
Compute the node proximity matrix for a graph $G$\;
\label{code:comp_prox}
Divide the graph $G$ into sets of disjoint subgraphs $G_\mathcal{S} = \{\mathcal{S}_1, \mathcal{S}_2, \cdots, \mathcal{S}_{|E|}\}$ based on Algorithm~\ref{Alg:edge_samp}\;
\label{code:graph_subset}
Initialize the weight matrices $\mathbf{W}_{in}$ and $\mathbf{W}_{out}$\;
\label{code:weight_init}
\For{$epoch = 0$; $epoch < n^{epoch}$}
{
    \tcp{Generate samples}
    Sample $B$ subgraphs uniformly at random from $G_\mathcal{S}$\;\label{code:samp_edges}
    \tcp{Optimize weights}
    Update $\mathbf{W}_{in}$ according to
    $\widetilde{\nabla}_{\mathbf{v}_i}L_{nov}^B$ in Eq.\,(\ref{eq:nonZeroNoiseGra_on_vi})\;
    \label{code:update_w_in}
    Update $\mathbf{W}_{out}$ according to
    $\widetilde{\nabla}_{\mathbf{v}_j}L_{nov}^B$ in Eq.\,(\ref{eq:nonZeroNoiseGra_on_vi})\;
    \label{code:update_w_out}
    \tcp{Update privacy accountant of RDP}
    Calculate RDP with sampling probability $\frac{B}{|E|}$\;
    \label{code:cal_RDP}
    $\hat{\delta}\leftarrow$ get privacy spent given the target $\epsilon$\;
    Stop optimization if $\hat{\delta} \geq \delta$\;
    \label{code:stop_opt}
}
\Return $\mathbf{W}_{in}$, $\mathbf{W}_{out}$\;
\end{algorithm}
\section{Privacy and Complexity Analysis}\label{subsec:priv_complex}
\subsection{Privacy Analysis}\label{sec:priv_analy}
% \textbf{Amplification by Subsampling.} 
Subsampling introduces a non-zero probability of an added or modified sample not to be processed by the randomized algorithm. Random sampling will enhance privacy protection and reduce privacy loss~\cite{wang2019subsampled,zhu2019poission,mironov2019r}. 
In this paper, we focus on the ``subsampling without replacement" setup, which adheres to the following privacy amplification theorem for $(\epsilon,\delta)$-DP.

\begin{definition}[Subsample~\cite{wang2019subsampled}]
Given a dataset $X$ containing $n$ points, the subsample procedure selects a random sample from the uniform distribution over all subsets of $X$ with size $m$. The ratio $\gamma=m/n$ is termed as the sampling parameter of the subsample procedure.
\end{definition}

\begin{theorem}[RDP for Subsampled Mechanisms~\cite{wang2019subsampled}]\label{theo:subsample}
Given a dataset of $n$ points drawn from a domain $\mathcal{X}$ and a mechanism $\mathcal{A}$ that accepts inputs from $\mathcal{X}^m$ for $m \leq n$, we consider the randomized algorithm $\mathcal{A}$ for subsampling, which is defined as follows: (1) sample $m$ data points without replacement from the dataset, where the sampling parameter is $\gamma = m / n$, and (2) apply $\mathcal{A}$ to the subsampled dataset. For all integers $\alpha \geq 2$, if $\mathcal{A}$ satisfies $(\alpha, \epsilon(\alpha))$ RDP, then the subsampled mechanism $\mathcal{A} \circ \rm{subsample}$ satisfies $\left(\alpha, \epsilon^{\prime}(\alpha)\right)$-RDP in which
\begin{equation*}\label{eq:RDP_subsamp}
\begin{split}
\epsilon^{\prime}(\alpha) \leq \frac{1}{\alpha-1} \log \big(1+\gamma^2\big(\begin{array}{l}
\alpha \\
2
\end{array}\big) \min \big\{4\big(e^{\epsilon(2)}-1\big), \\
e^{\epsilon(2)} \min \big\{2,\big(e^{\epsilon(\infty)}-1\big)^2\big\}\big\} \\
+\sum_{j=3}^\alpha \gamma^j\big(\begin{array}{l}
\alpha \\
j
\end{array}\big) e^{(j-1) \epsilon(j)} \min \big\{2,\big(e^{\epsilon(\infty)}-1\big)^j\big\}\big).
\end{split}
\end{equation*}
\end{theorem}

Following Theorem~\ref{theo:subsample}, we conduct a privacy analysis of the generated $\mathbf{W}_{in}$ and $\mathbf{W}_{out}$ in Algorithm~\ref{alg:PrivSkipGram}. We adopt the functional perspective of RDP, where $\epsilon$ is a function of $\alpha$, with $1<\alpha<\infty$, and this function is determined by the private algorithm. For ease of presentation, we replace $\left(\alpha, \epsilon^\prime(\alpha)\right)$ with $\left(\alpha, \epsilon^\gamma(\alpha)\right)$ in the following proof, where $\gamma$ denotes the sampling probability.

\begin{theorem}\label{PrivTheorem}
Given the batch size $B$ and the total number of edges $|E|$ in the graph, the outputs of Algorithm~\ref{alg:PrivSkipGram} satisfy node-level
$\big(\alpha, n^{epoch}\epsilon^{\frac{B}{|E|}}(\alpha)\big)$-RDP after completing $n^{epoch}$ training iterations.
\end{theorem}

\begin{proof}
Without considering privacy amplification, every iteration $epoch$ of Algorithm~\ref{alg:PrivSkipGram} is $(\alpha, \epsilon(\alpha))$-RDP where $\epsilon(\alpha)=\frac{\alpha \cdot S_{\nabla_{\mathbf{v}}}^2}{2 \sigma^2}$, which can be directly derived from Corollary 3 in~\cite{mironov2017renyi}. Since we take $B$ subgraphs uniformly at random from $G_\mathcal{S}$, by setting $\gamma$ to $\frac{B}{|E|}$, the privacy guarantee of the outputs of Algorithm~\ref{alg:PrivSkipGram} can be directly analyzed by Theorem~\ref{theo:subsample}, yielding $\left(\alpha, \epsilon^\gamma(\alpha)\right)$-RDP for every $epoch$. After $n^{epoch}$ iterations, the outputs of Algorithm~\ref{alg:PrivSkipGram} satisfy node-level
$\big(\alpha, n^{epoch}\epsilon^{\frac{B}{|E|}}(\alpha)\big)$-RDP. Finally, we apply the conversion rule in Theorem~\ref{Theo:RDP_to_DP} to convert the RDP back to DP.
\end{proof}

\subsection{Complexity Analysis}
We analyze the computational complexity of each step of SE-PrivGEmb. 
DeepWalk proximity~\cite{yang2015network} and node degree proximity are used in experiments. Computing DeepWalk proximity~\cite{yang2015network} takes $\mathcal{O}(|V|^2)$ time, and computing node degree proximity takes $\mathcal{O}(|V|)$ time.
The time complexity of dividing the graph $G$ into sets of subgraphs is $|E|k$. 
Initializing the weight matrices $\mathbf{W}_{in}$ and $\mathbf{W}_{out}$ is constant time $\mathcal{O}(1)$. 
The time complexity of updating $\mathbf{W}_{in}$ depends on the embedding dimension $r$ and the number of batch samples, which can be approximated as $\mathcal{O}(rB)$. Similarly, the time complexity of updating $\mathbf{W}_{out}$ can be approximated as $\mathcal{O}(rB)$. The time complexity of updating the DP cost using RDP depends on the specific implementation of RDP. Different versions of RDP may result in slight differences in time complexity, but according to~\cite{bu2023differentially}, these implementations all have asymptotic complexity of $\mathcal{O}(rB\gamma)$, where $\gamma$ denotes the sampling probability. After $n^{epoch}$ epochs, the total time complexity is $\mathcal{O}(|V|^2 + |E|k + n^{epoch}rB + n^{epoch}rB\gamma)$ when DeepWalk proximity is used. Alternatively, when node degree proximity is used, the total time complexity becomes $\mathcal{O}(|V| + |E|k + n^{epoch}rB + n^{epoch}rB\gamma)$.  

\section{Experiments}\label{Exp_Section}
In this section, we evaluate the performance of SE-PrivGEmb in two downstream tasks: structural equivalence and link prediction. The former demonstrates the ability of our proposed method to recover structural equivalence. It refers to a concept wherein two nodes within a graph are deemed structurally equivalent if they possess identical connections with the same nodes, \emph{which is widely acknowledged as the simplest and most rigorous metric}~\cite{jin2021toward}.
On the other hand, link prediction is a commonly used benchmark task in graph learning models, which is used to further display the performance of SE-PrivGEmb. Our specific goal is to answer the following four questions:

\begin{itemize}
\item How much do the parameters impact SE-PrivGEmb's performance? (See Section~\ref{exp:para_impact})
\item How much do the perturbation strategies impact SE-PrivGEmb's performance? (See Section~\ref{exp:perturb_strategy})
\item How much does the privacy budget affect SE-PrivGEmb's performance in structural equivalence? (See Section~\ref{exp:struc_equ})
\item How much does the privacy budget affect SE-PrivGEmb's performance in link prediction? (See Section~\ref{exp:link_pred})
\end{itemize}

\subsection{Experimental Setup}
\textbf{Datasets.}
We run experiments on six real-world datasets, Chameleon, PPI, Power, Arxiv, BlogCatalog, and DBLP.  
Since we focus on simple graphs in this work, all datasets are pre-processed to remove self-loops. The details of the datasets are as follows.

\begin{itemize}
\item \textbf{Chamelon\footnote{\url{https://snap.stanford.edu/data/wikipedia-article-networks.html}}}.
The dataset is collected from the English Wikipedia on the chamelon topic with 2,277 nodes and 31,421 edges, where nodes represent articles and edges indicate mutual links between them.
\item \textbf{PPI}~\cite{stark2006biogrid}.
This is a human Protein-Protein Interaction network with 3,890 nodes and 76,584 edges, where nodes represent proteins and edges indicate interactions between these proteins.
\item \textbf{Power\footnote{\url{http://konect.cc/networks/opsahl-powergrid/}}}.
This dataset represents an electrical grid that stretches across the western United States with 4,941 nodes and 6,594 edges. The nodes represent the buses within the grid, and the edges signify the transmission lines that connect these buses.
\item \textbf{Arxiv\footnote{\url{https://snap.stanford.edu/data/ca-GrQc.html}}}.
This network with 5,242 nodes and 14,496 edges is derived from the e-print arXiv and specifically examines scientific collaborations among authors who submit papers to the Astrophysics category.
\item \textbf{BlogCatalog\footnote{\url{http://datasets.syr.edu/datasets/BlogCatalog3.html}}}.
This network visualizes the intricate web of social interactions among online bloggers with 10,312 nodes and 333,983 edges, where nodes represent individual bloggers and edges denote the connections or relationships between them.
\item \textbf{DBLP\footnote{\url{https://www.aminer.cn/citation}}}.
This dataset illustrates a scholarly network with 2,244,021 nodes and 4,354,534 edges in which nodes represent papers, authors, and venues, while edges indicate authorship relationships and the publication venues of the papers.
\end{itemize}

\textbf{Evaluation Metrics.}
we measure the performance of SE-PrivGEmb by examining its effectiveness in two downstream tasks: structural equivalence and link prediction.
\begin{itemize}
    \item For structural equivalence, two nodes are considered to be structurally equivalent if they have many of the same neighbors. To evaluate the ability of an embedding method to recover structural equivalence, we introduce a distance metric. The distance $dist(\mathbf{A}_i, \mathbf{A}_j)$ is defined as the difference between the lines of the adjacency matrix for each pair of nodes, $(v_i, v_j)$, whereas $dist(\mathbf{Y}_i, \mathbf{Y}_j)$ represents the distance between their corresponding node vectors in the embedding space. To quantify structural equivalence, we calculate the correlation coefficient (i.e., Pearson) between these values for all node pairs. In this paper, we define $\text{StrucEqu} = pearson(dist(\mathbf{A}_i, \mathbf{A}_j), dist(\mathbf{Y}_i, \mathbf{Y}_j))$, where the Euclidean distance is used.
    \item For the link prediction task, as with~\cite{zhang2018link}, the existing links in each dataset are randomly divided into a training set (90\%) and a test set (10\%).
    To evaluate the performance of link prediction, we randomly select an equal number of node pairs without connected edges as negative test links for the test set. Additionally, for the training set, we sample the same number of node pairs without edges to construct negative training data. The performance metric is the area under the ROC curve (AUC).
\end{itemize}

We measure each result over ten experiments and report the average value. A higher value of either \textbf{StrucEqu} or \textbf{AUC} indicates better utility.

\textbf{Competitive Methods.}
In this paper, we design two PrivEmb methods\footnote{Our code is available at \url{https://github.com/sunnerzs/SEPrivGEmb}.} fusing DeepWalk proximity~\cite{yang2015network} and node degree proximity: \textbf{SE-PrivGEmb$_{DW}$} and \textbf{SE-PrivGEmb$_{Deg}$}. 
To establish a baseline for comparison, we utilize four state-of-the-art private graph learning methods, namely
DPGGAN~\cite{yang2020secure},
DPGVAE~\cite{yang2020secure},
GAP~\cite{sajadmanesh2023gap},
and ProGAP~\cite{sajadmanesh2023progap}.
Moreover, we also use SE-GEmb$_{DW}$ and SE-GEmb$_{Deg}$ as non-private counterparts to SE-PrivGEmb$_{DW}$ and SE-PrivGEmb$_{Deg}$, respectively, to better show the performance of SE-PrivGEmb.
In this study, we simulate a scenario where the graphs only contain structural information, while GAP and ProGAP rely on node features. To ensure a fair evaluation, similar to prior research~\cite{du2022understanding}, we use randomly generated features as inputs for both methods.

\textbf{Parameter Settings.}
To demonstrate the performance of our proposed method across different downstream tasks, we set the same parameters for structural equivalence and link prediction except for the training epochs. Specifically, we set the training epochs to \text{$n^{epoch}=200$} for structural equivalence and \text{$n^{epoch}=2000$} for link prediction. The embedding dimension is fixed at \text{$r=128$}. It is worth noting that we do not specifically highlight the impact of $r$ as it is commonly used in various network embedding methods~\cite{perozzi2014deepwalk,tang2015line,lai2017prune,tu2018unified}.
For the privacy parameters, as with~\cite{yang2020secure}, we fix \text{$\delta=10^{-5}$} and \text{$\sigma=5$}, and vary the privacy budget $\epsilon$ among the values $\{0.5, 1, 1.5, 2, 2.5, 3, 3.5\}$ for the private methods. Additionally, we vary the batch size $B$, learning rate $\eta$, and the parameters $C$ and $k$ to verify their impact on the utility of SE-PrivGEmb. To ensure consistency with the original papers, we utilize the official GitHub implementations for DPGGAN, DPGVAE, GAP and ProGAP, and replicate the experimental setups described in those papers.

\subsection{Impact of Parameters}\label{exp:para_impact}
In this section, we vary the batch size $B$, learning rate $\eta$, and the parameters $C$ and $k$ to determine suitable values for the structural equivalence task. Given the privacy budget \text{$\epsilon=3.5$}, our experiments are conducted on three datasets: Chameleon, Power, and Arxiv.

\subsubsection{\textbf{Parameter $B$}} 
In this experiment, we investigate the impact of the parameter $B$ in Table~\ref{Tab:impact_of_batchSize}. For SE-PrivGEmb$_{DW}$, the results indicate that for the Chameleon dataset, \text{$B=64$} is optimal, achieving the highest StrucEqu value of 0.4599 with a relatively small standard deviation (SD) value of 0.0530. In the case of the Power dataset, \text{$B=128$} is recommended, as it produces a higher StrucEqu value of 0.2522, with a comparable SD value of 0.0543. Similarly, on the Arxiv dataset, \text{$B=128$} performs well, yielding a StrucEqu value of 0.3457 and an SD value of 0.0303. For SE-PrivGEmb$_{Deg}$, we can observe that on the Chameleon dataset, \text{$B=1024$} emerges as the best choice, achieving a StrucEqu value of 0.3967 and an SD value of 0.0535. On the Power dataset, \text{$B=128$} yields a StrucEqu value of 0.2322 and an SD value of 0.0476. On the Arxiv dataset, \text{$B=128$} remains suitable, achieving a StrucEqu value of 0.2341 and an SD value of 0.0193. In summary, for both SE-PrivGEmb$_{DW}$ and SE-PrivGEmb$_{Deg}$, \text{$B=128$} consistently proves to be a suitable choice. Therefore, we set the batch size $B$ to 128 in the subsequent experiments.
\begin{table}[]
\centering
\caption{Summary of StrucEqu values with different $B$, given $\epsilon=3.5$ \\
(Result: average StrucEqu $\pm$ SD; \textbf{Bold}: best)}
\label{Tab:impact_of_batchSize}
\begin{tabular}{c|ccc}
\hline
                      & \multicolumn{3}{c}{SE-PrivGEmb$_{DW}$} \\ \cline{2-4}
\multirow{-2}{*}{$B$} & Chameleon               & Power               & Arxiv                    \\ \hline
32                    & 0.4281$\pm$0.0529 & 0.1549$\pm$0.0196 & 0.3375$\pm$0.0640
\\
64                    & \textbf{0.4599$\pm$0.0530} & 0.1815$\pm$0.0327 & 0.3376$\pm$0.0871
\\
% \rowcolor[HTML]{CBCEFB}
128                   & 0.4507$\pm$0.0341 & \textbf{0.2522$\pm$0.0543} &  \textbf{0.3457$\pm$0.0303} \\
256                   & 0.3701$\pm$0.0181 & 0.2053$\pm$0.0670 & 0.3326$\pm$0.0564 \\
512                   & 0.3293$\pm$0.0324 & 0.2462$\pm$0.0508 & 0.3411$\pm$0.0031 \\
1024                  & 0.2780$\pm$0.0406 & 0.2520$\pm$0.0157 & 0.3416$\pm$0.0161 \\ \hline
                      & \multicolumn{3}{c}{SE-PrivGEmb$_{Deg}$} \\ \cline{2-4}
\multirow{-2}{*}{$B$} & Chameleon               & Power               & Arxiv        \\ \hline
32                    & 0.2611$\pm$0.0714 & 0.1127$\pm$0.0124 & 0.1394$\pm$0.0815 \\
64                    & 0.3084$\pm$0.0172 & 0.1216$\pm$0.0581 & 0.1625$\pm$0.0722 \\
% \rowcolor[HTML]{CBCEFB}
128                   & 0.3661$\pm$0.0092 & \textbf{0.2322$\pm$0.0476} & \textbf{0.2341$\pm$0.0193} \\
256                   & 0.3956$\pm$0.0090 & 0.1595$\pm$0.0418 & 0.2260$\pm$0.0292 \\
512                   & 0.3710$\pm$0.0834 & 0.1833$\pm$0.0303 & 0.2204$\pm$0.0656 \\
1024                  & \textbf{0.3967$\pm$0.0535} & 0.1773$\pm$0.0317 & 0.2251$\pm$0.0509 \\ \hline
\end{tabular}
\end{table}

\subsubsection{\textbf{Parameter $\eta$}} 
In this experiment, we investigate the impact of the learning rate $\eta$, as summarized in Table~\ref{Tab:impact_of_learnRate}.
For SE-PrivGEmb$_{DW}$, on the Chameleon dataset, \text{$\eta=0.15$} achieves the highest StrucEqu value of 0.4573, with a small SD value of 0.0072. The optimal learning rate for the Power dataset is \text{$\eta=0.1$}, which yields a StrucEqu value of 0.2522 and an SD value of 0.0543. \text{$\eta=0.25$} performs well on the Arxiv dataset, with a StrucEqu value of 0.4353 and an SD value of 0.0198.
For SE-PrivGEmb$_{Deg}$, on the Chameleon dataset, \text{$\eta = 0.1$} is identified as the best choice, achieving a StrucEqu value of 0.3661 and an SD value of 0.0092. On the Power dataset, \text{$\eta = 0.1$} is recommended, resulting in a StrucEqu value of 0.2322 and an SD value of 0.0476. On the Arxiv dataset, \text{$\eta = 0.1$} remains suitable, achieving a StrucEqu value of 0.2341, with an SD value of 0.0193.
Overall, across both SE-PrivGEmb$_{DW}$ and SE-PrivGEmb$_{Deg}$, \text{$\eta = 0.1$} emerges as a consistently suitable choice. It consistently yields higher StrucEqu values across datasets while maintaining relatively small SD values.
\begin{table}[]
\centering
\caption{Summary of StrucEqu values with different $\eta$, given $\epsilon=3.5$ \\
(Result: average StrucEqu $\pm$ SD; \textbf{Bold}: best)}
\label{Tab:impact_of_learnRate}
% \centering
\begin{tabular}{c|ccc}
\hline
                         & \multicolumn{3}{c}{SE-PrivGEmb$_{DW}$}                                        \\ \cline{2-4}
\multirow{-2}{*}{$\eta$} & Chameleon               & Power               & Arxiv                          \\ \hline
0.01                     & 0.0703$\pm$0.0378 & 0.0515$\pm$0.0180 & 0.0436$\pm$0.0121 \\
0.05                     & 0.4290$\pm$0.0045 & 0.1756$\pm$0.0241 & 0.3989$\pm$0.0311 \\
% \rowcolor[HTML]{CBCEFB}
0.1                      & 0.4507$\pm$0.0341& \textbf{0.2522$\pm$0.0543} & 0.3457$\pm$0.0303  \\
0.15                     & \textbf{0.4573$\pm$0.0072} & 0.2431$\pm$0.0198 & 0.3938$\pm$0.0204 \\
0.2                      & 0.4152$\pm$0.0470 & 0.1912$\pm$0.0070 & 0.4282$\pm$0.0285 \\
0.25                     & 0.4564$\pm$0.0112 & 0.2478$\pm$0.0546 & \textbf{0.4353$\pm$0.0198} \\
0.3                      & 0.4524$\pm$0.0262 & 0.2213$\pm$0.0342 & 0.4011$\pm$0.0154 \\ \hline

                         & \multicolumn{3}{c}{SE-PrivGEmb$_{Deg}$}                                         \\ \cline{2-4}
\multirow{-2}{*}{$\eta$} & Chameleon               & Power               & Arxiv                 \\ \hline
0.01                      & 0.0409$\pm$0.0148 & 0.0266$\pm$0.0331 & 0.0122$\pm$0.0292 \\
0.05                     & 0.3221$\pm$0.0419 & 0.1524$\pm$0.0531 & 0.1876$\pm$0.0445 \\
% \rowcolor[HTML]{CBCEFB}
0.1                      & \textbf{0.3661$\pm$0.0092} & \textbf{0.2322$\pm$0.0476} & \textbf{0.2341$\pm$0.0193} \\
0.15                     & 0.3503$\pm$0.0632 & 0.1742$\pm$0.0598 & 0.1725$\pm$0.0320 \\
0.2                      & 0.3169$\pm$0.0529 & 0.1372$\pm$0.0723 & 0.2189$\pm$0.0445 \\
0.25                     & 0.3656$\pm$0.0552 & 0.1435$\pm$0.0595 & 0.2335$\pm$0.0464 \\
0.3                      & 0.2628$\pm$0.0471 & 0.1229$\pm$0.0085 & 0.2086$\pm$0.0765 \\ \hline
\end{tabular}
\end{table}

\subsubsection{\textbf{Parameter $C$}}
In this experiment, we investigate the impact of the gradient clipping threshold $C$, as shown in Table~\ref{Tab:impact_of_gradClip}.
For SE-PrivGEmb$_{DW}$, on the Chameleon dataset, the highest StrucEqu value of 0.4507 is achieved at \text{$C=2$}, with a relatively low SD value of 0.0341, indicating stable results. Other strong contenders are \text{$C=3$} and \text{$C=4$}, though their StrucEqu values are slightly lower than those at \text{$C=2$}. On the Power dataset, the highest StrucEqu value of 0.2522 appears at \text{$C=2$}, with an SD value of 0.0543, showing superior performance at this setting. While \text{$C=1$} and \text{$C=4$} also exhibit relatively high StrucEqu values, \text{$C=2$} remains notably preferable. On the Arxiv dataset, \text{$C=1$} and \text{$C=5$} show slightly higher StrucEqu values than other $C$ values while maintaining stability. In addition, \text{$C=2$} and \text{$C=6$} are considered balanced options in terms of both StrucEqu and SD. 
For SE-PrivGEmb$_{Deg}$, on the Chameleon dataset, \text{$C=2$} yields the highest StrucEqu value of 0.3661, with an extremely low SD value of 0.0092, indicating highly stable results. Other $C$ values show lower StrucEqu values yet higher SD values, making \text{$C=2$} the optimal choice. On the Power dataset, \text{$C=2$} performs best with the highest StrucEqu value of 0.2322 and a moderate SD value of 0.0476. Although \text{$C=1$} and \text{$C=4$} exhibit relatively high StrucEqu values, \text{$C=2$} remains the preferred option. On the Arxiv dataset, \text{$C=1$} and \text{$C=2$} show favorable combinations of StrucEqu and SD, particularly \text{$C=2$} with a StrucEqu value of 0.2341 and an SD value of 0.0193. In summary, \text{$C=2$} emerges as a suitable choice for both SE-PrivGEmb$_{DW}$ and SE-PrivGEmb$_{Deg}$, consistently delivering high structure preservation and stability across different datasets.

\begin{table}[]
\centering
\caption{Summary of StrucEqu values with different $C$, given $\epsilon=3.5$ \\
(Result: average StrucEqu $\pm$ SD; \textbf{Bold}: best)}
\label{Tab:impact_of_gradClip}
\begin{tabular}{c|ccc}
\hline
                      & \multicolumn{3}{c}{SE-PrivGEmb$_{DW}$}      \\ \cline{2-4}
\multirow{-2}{*}{$C$} & Chameleon               & Power               & Arxiv                              \\ \hline
1                     & 0.3248$\pm$0.0595 & 0.2216$\pm$0.0281 & \textbf{0.4266$\pm$0.0588} \\
% \rowcolor[HTML]{CBCEFB}
2                     & \textbf{0.4507$\pm$0.0341} & \textbf{0.2522$\pm$0.0543} & 0.3457$\pm$0.0303  \\
3                     & 0.4319$\pm$0.0475 & 0.1774$\pm$0.0407 & 0.3514$\pm$0.0219 \\
4                     & 0.4368$\pm$0.0214 & 0.1972$\pm$0.0414 & 0.3555$\pm$0.0142 \\
5                     & 0.4079$\pm$0.0172 & 0.1652$\pm$0.0345 & 0.3671$\pm$0.0352 \\
6                     & 0.3719$\pm$0.0328 & 0.1710$\pm$0.0557 & 0.3476$\pm$0.0305 \\ \hline

                      & \multicolumn{3}{c}{SE-PrivGEmb$_{Deg}$}        \\ \cline{2-4}
\multirow{-2}{*}{$C$} & Chameleon               & Power               & Arxiv                 \\ \hline
1                     & 0.2485$\pm$0.0757 & 0.2030$\pm$0.0846 & 0.2206$\pm$0.0858 \\
% \rowcolor[HTML]{CBCEFB}
2                     & \textbf{0.3661$\pm$0.0092} & \textbf{0.2322$\pm$0.0476} & \textbf{0.2341$\pm$0.0193} \\
3                     & 0.2907$\pm$0.0921 & 0.1140$\pm$0.0308 & 0.1819$\pm$0.0767 \\
4                     & 0.3116$\pm$0.0750 & 0.1820$\pm$0.0541 & 0.2107$\pm$0.0450 \\
5                     & 0.3384$\pm$0.0892 & 0.1731$\pm$0.0260 & 0.1464$\pm$0.0531 \\
6                     & 0.3136$\pm$0.0699 & 0.1763$\pm$0.0482 & 0.2143$\pm$0.0269 \\ \hline
\end{tabular}
\end{table}

\subsubsection{\textbf{Parameter $k$}} 
We also explore the impact of the negative sampling number $k$ within $\{1,2,3,4,5,6,7\}$, as shown in Table~\ref{Tab:impact_of_k}. For SE-PrivGEmb$_{DW}$, all SD values demonstrate comparable performance across various datasets. For the StrucEqu value, on the Chameleon dataset, \text{$k=5$} achieves the highest StrucEqu value of 0.4507. On the Power dataset, StrucEqu values fluctuate across different $k$ values but also peak at \text{$k=5$}. On the Arxiv dataset, the highest StrucEqu value is observed at \text{$k=7$}, with \text{$k=5$} being close to this peak. For SE-PrivGEmb$_{Deg}$, on the Chameleon dataset, \text{$k=6$} achieves a high StrucEqu value of 0.3824, yet it also derives the highest SD value of 0.0574, indicating relatively poor stability. Notably, \text{$k=5$} achieves a StrucEqu value that is competitive with \text{$k=6$}, while maintaining a significantly lower SD value of 0.0092. On the Power dataset, SE-PrivGEmb$_{Deg}$ achieves a significant peak at \text{$k=5$} in terms of StrucEqu value. On the Arxiv dataset, the highest StrucEqu value is observed at \text{$k=7$}, but \text{$k=5$} also demonstrates competitive performance with a lower SD value compared to \text{$k=7$}. Thus, taking into account the performance across various datasets, selecting $k=5$ is a well-balanced choice, as it reliably achieves strong structure preservation while ensuring stability.

\begin{table}[]
\centering
\caption{Summary of StrucEqu values with different $k$, given $\epsilon=3.5$ \\
(Result: average StrucEqu $\pm$ SD; \textbf{Bold}: best)}
\label{Tab:impact_of_k}
\begin{tabular}{c|ccc}
\hline
                      & \multicolumn{3}{c}{SE-PrivGEmb$_{DW}$}        \\ \cline{2-4}
\multirow{-2}{*}{$k$} & Chameleon               & Power                & Arxiv                      \\ \hline
1                     & 0.4024$\pm$0.0249 & 0.2036$\pm$0.0766 & 0.3423$\pm$0.0228  \\
2                     & 0.4121$\pm$0.0339 & 0.2033$\pm$0.0249 & 0.3500$\pm$0.0450  \\
3                     & 0.4325$\pm$0.0304 & 0.1913$\pm$0.0132 & 0.3126$\pm$0.0447 \\
4                     & 0.4190$\pm$0.0670 & 0.2441$\pm$0.0177 & 0.3801$\pm$0.0421  \\
% \rowcolor[HTML]{CBCEFB}
5                     & \textbf{0.4507$\pm$0.0341} & \textbf{0.2522$\pm$0.0543} & 0.3457$\pm$0.0303  \\
6                     & 0.4053$\pm$0.0501 & 0.2108$\pm$0.0619 & 0.3599$\pm$0.0456  \\
7                     & 0.3916$\pm$0.0692 & 0.2010$\pm$0.0565 & \textbf{0.4247$\pm$0.0391}  \\ \hline

                      & \multicolumn{3}{c}{SE-PrivGEmb$_{Deg}$}          \\ \cline{2-4}
\multirow{-2}{*}{$k$} & Chameleon               & Power               & Arxiv                 \\ \hline
1                     & 0.2286$\pm$0.0121 & 0.1425$\pm$0.0279 & 0.2203$\pm$0.0587 \\
2                     & 0.2338$\pm$0.0450 & 0.1542$\pm$0.0202 & 0.2229$\pm$0.0673 \\
3                     & 0.2406$\pm$0.0531 & 0.0839$\pm$0.0752 & 0.2068$\pm$0.0654 \\
4                     & 0.2371$\pm$0.0717 & 0.1987$\pm$0.0714 & 0.1993$\pm$0.0462 \\
% \rowcolor[HTML]{CBCEFB}
5                     & 0.3661$\pm$0.0092 & \textbf{0.2322$\pm$0.0476} & 0.2341$\pm$0.0193 \\
6                     & \textbf{0.3824$\pm$0.0574} & 0.1762$\pm$0.0183 & 0.2232$\pm$0.0076 \\
7                     & 0.3710$\pm$0.0375 & 0.1354$\pm$0.0441 & \textbf{0.2492$\pm$0.0432} \\ \hline
\end{tabular}
\end{table}

\begin{table}[]
\centering
\caption{Impact of Perturbation Strategies in Structural Equivalence \\
(Result: average StrucEqu $\pm$ SD; \textbf{Bold}: best)}
\label{tab:impact_of_Pert}
\begin{tabular}{l|llll}
\hline
\multicolumn{1}{c|}{\multirow{2}{*}{Datasets}} & \multicolumn{2}{c}{SE-PrivGEmb$_{DW}$}                   \\ \cline{2-3} 
                          & \multicolumn{1}{c}{Naive} & \multicolumn{1}{c}{Non-zero} \\ \hline
Chameleon($\epsilon=0.5$)                          & 0.1169$\pm$0.0010       & \textbf{0.3927$\pm$0.0341}         \\
Chameleon($\epsilon=2$)                            & 0.1192$\pm$0.0281       & \textbf{0.4177$\pm$0.0359}        \\
Chameleon($\epsilon=3.5$)                          & 0.1224$\pm$0.0159       &  \textbf{0.4507$\pm$0.0341}         \\ \hline
Power($\epsilon=0.5$)                          & 0.0857$\pm$0.0164      &  \textbf{0.1945$\pm$0.0567}          \\
Power($\epsilon=2$)                            & 0.0788$\pm$0.0227     & \textbf{0.2003$\pm$0.0330}          \\
Power($\epsilon=3.5$)                          & 0.0796$\pm$0.0160       & \textbf{0.2522$\pm$0.0543}       \\ \hline
Arxiv($\epsilon=0.5$)                            & 0.0888$\pm$0.0033       & \textbf{0.1770$\pm$0.0756}          \\
Arxiv($\epsilon=2$)                              & 0.0844$\pm$0.0263     & \textbf{0.2199$\pm$0.0171}            \\
Arxiv($\epsilon=3.5$)                            & 0.0848$\pm$0.0086       & \textbf{0.3457$\pm$0.0303}     \\ \hline

\multicolumn{1}{c|}{\multirow{2}{*}{Datasets}} & \multicolumn{2}{c}{SE-PrivGEmb$_{Deg}$}                \\ \cline{2-3} 
                         & \multicolumn{1}{c}{Naive} & \multicolumn{1}{c}{Non-zero}  \\ \hline
Chameleon($\epsilon=0.5$)                          & 0.1119$\pm$0.0315       & \textbf{0.2914$\pm$0.0188}   \\
Chameleon($\epsilon=2$)                            & 0.1144$\pm$0.0094       & \textbf{0.3392$\pm$0.0367}    \\
Chameleon($\epsilon=3.5$)                          & 0.1258$\pm$0.0330      & \textbf{0.3661$\pm$0.0092}     \\ \hline
Power($\epsilon=0.5$)                          & 0.0638$\pm$0.0013       & \textbf{0.0877$\pm$0.0261}     \\
Power($\epsilon=2$)                            & 0.0852$\pm$0.0094       &  \textbf{0.1542$\pm$0.0231}    \\
Power($\epsilon=3.5$)                          & 0.0833$\pm$0.0151       & \textbf{0.2322$\pm$0.0476}     \\ \hline
Arxiv($\epsilon=0.5$)                            & 0.0779$\pm$0.0267      & \textbf{0.1629$\pm$0.0381}       \\
Arxiv($\epsilon=2$)                              & 0.0611$\pm$0.0130       & \textbf{0.2112$\pm$0.0768}       \\
Arxiv($\epsilon=3.5$)                            & 0.0784$\pm$0.0168      & \textbf{0.2341$\pm$0.0193}      \\ \hline
\end{tabular}
\end{table}

\subsection{Impact of Perturbation Strategies}\label{exp:perturb_strategy}
Table~\ref{tab:impact_of_Pert} presents the results of the structural equivalence task using different noise addition strategies, namely naive perturbation using Eq.\,(\ref{eq:novNoiseGra_on_vi}) and non-zero perturbation using Eq.\,(\ref{eq:nonZeroNoiseGra_on_vi}), under different privacy budgets \text{$\epsilon\in\{0.5, 2, 3.5\}$}. 
Across all three datasets, the non-zero perturbation strategy consistently outperforms the naive perturbation strategy, indicating that our proposed methods (SE-PrivGEmb$_{DW}$ and SE-PrivGEmb$_{Deg}$) are more effective in preserving structural equivalence.
As $\epsilon$ increases, the utility of the results generally improves under both strategies. However, the non-zero perturbation strategy maintains a significant advantage over the naive perturbation strategy, even at very large privacy budgets, demonstrating its robustness and effectiveness in maintaining data utility while ensuring privacy.
\begin{figure*}[]
  \centerline{
  \includegraphics[width=7in]{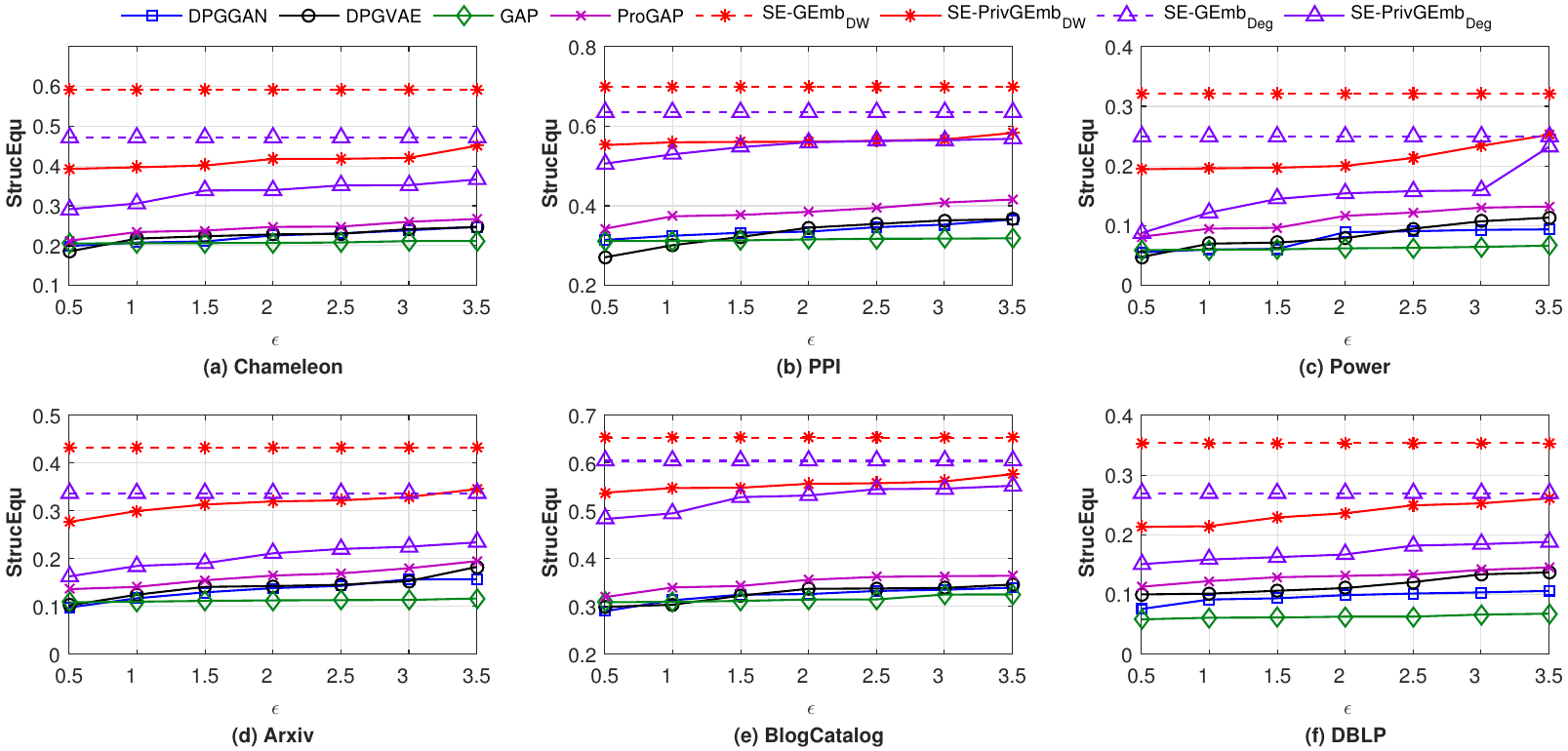}
  }
  \caption{Impact of Privacy Budget on Structural Equivalence}
  \label{fig:struc}
\end{figure*}

\begin{figure*}[]
  \centerline{
  \includegraphics[width=7in]{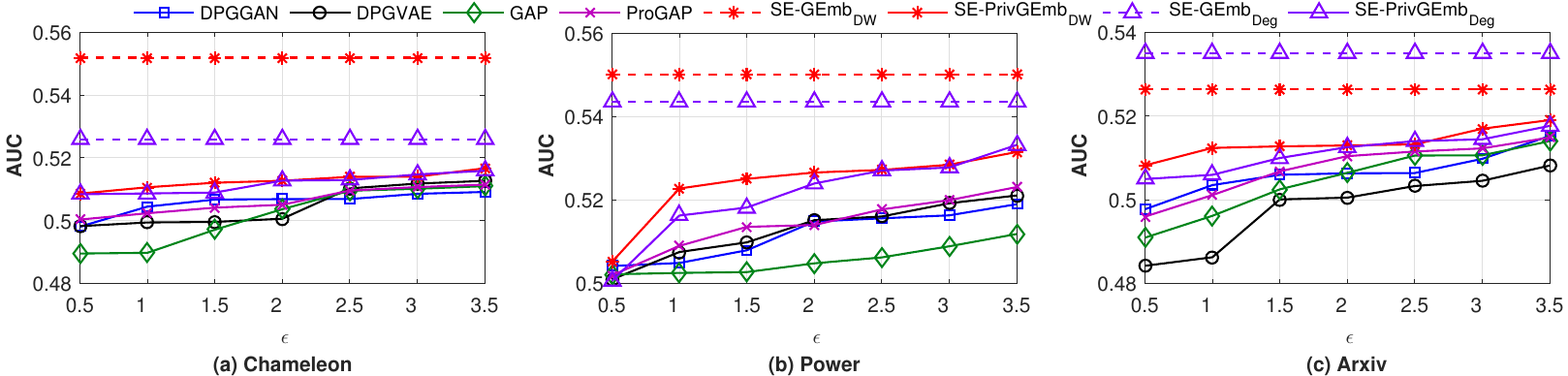}
  }
  \caption{Impact of Privacy Budget on Link Prediction}
  \label{fig:lp}
\end{figure*} 

\subsection{Impact of Privacy Budget on Structural Equivalence}\label{exp:struc_equ}
In the context of DP, the privacy budget $\epsilon$ is a crucial parameter for determining privacy levels. We compare the StrucEqu results of various methods across seven privacy budgets from 0.5 to 3.5. Fig.\,\ref{fig:struc} illustrates the StrucEqu outcomes for structural equivalence derived from eight algorithms. From this figure, we observe a consistent trend: utility generally improves with higher privacy budgets. In particular, SE-PrivGEmb$_{DW}$ and SE-PrivGEmb$_{Deg}$ consistently achieve the highest accuracy in terms of StrucEqu among all privacy-preserving algorithms and maintain competitive performance compared to the non-private SE-GEmb$_{DW}$ and SE-GEmb$_{Deg}$. DPGGAN and DPGVAE frequently yield poor results, primarily because they tend to converge prematurely when using MA, especially when the privacy budget is small. GAP perturbs the aggregate information, and the aggregation perturbation encounters compatibility issues with GNNs. Consequently, all aggregate outputs need to be re-perturbed at each training iteration, resulting in poor performance. ProGAP offers slightly better utility than GAP, and yet its performance remains significantly lower than that of SE-PrivGEmb$_{DW}$ and SE-PrivGEmb$_{Deg}$, particularly at small privacy budgets.

\subsection{Impact of Privacy Budget on Link Prediction}\label{exp:link_pred}
Fig.\,\ref{fig:lp} displays the link prediction results for six private algorithms and two non-private algorithms across three datasets: Chameleon, Power, and Arxiv. It is evident that the non-private SE-GEmb$_{DW}$ and SE-GEmb$_{Deg}$ consistently outperform all private algorithms across all values of $\epsilon$ on three datasets. On the Chameleon dataset, SE-PrivGEmb$_{DW}$ and SE-PrivGEmb$_{Deg}$ outperform the remaining private algorithms across all values of $\epsilon$. DPGGAN and DPGVAE outperform GAP at \text{$\epsilon\in\{0.5,1,1.5\}$}, while DPGVAE and GAP excel against DPGGAN under \text{$\epsilon\in\{2.5,3,3.5\}$}. On the Power dataset, SE-PrivGEmb$_{DW}$ and SE-PrivGEmb$_{Deg}$ generally outperform other private algorithms across most values of $\epsilon$. All private algorithms yield similar AUC values at \text{$\epsilon=0.5$}, while SE-PrivGEmb$_{DW}$ and SE-PrivGEmb$_{Deg}$ outperform the remaining private algorithms under \text{$\epsilon\in\{1, 1.5, 2, 2.5, 3, 3.5\}$}. For the Arxiv dataset, SE-PrivGEmb$_{DW}$ and SE-PrivGEmb$_{Deg}$ outperform the other private algorithms for all values of $\epsilon$. It is worth noting that DPGGAN, GAP, and ProGAP show comparable performance, whereas DPGVAE consistently underperforms compared to the other private algorithms. In summary, our proposed SE-PrivGEmb$_{DW}$ and SE-PrivGEmb$_{Deg}$ exhibit strong performance in maintaining privacy while preserving predictive accuracy, making them promising approaches for link prediction under differential privacy.
\section{Related Work}\label{Related_work}
% The related work of this paper covers differentially private deep learning and differentially private graph embedding generation.
% \subsection{Differentially Private Deep Learning}
\noindent\textbf{Deep Learning with DP.}
Unlike traditional privacy-preserving methods~\cite{sweeney2002k, machanavajjhala2007diversity, hu2014private}, differential privacy (DP) and its variant, local differential privacy (LDP), offer strong privacy guarantees and robustness against adversaries with prior knowledge and have been widely applied across various fields~\cite{chaudhuri2011differentially, bassily2014private, zhang2018calm, ye2021beyond, zhang2021privsyn, ye2023stateful, Zhang2023trajectory}, particularly in privacy-preserving deep learning in recent years.
Nevertheless, applying DP to deep learning is challenging due to the complexity of deep neural networks (DNNs).
Previous research~\cite{mcmahan2017learning, song2013stochastic} has proposed incorporating norm-based gradient clipping into the SGD algorithm to mitigate this challenge. This technique limits the influence of individual examples on the gradients and introduces DP mechanisms to perturb the gradients accordingly. By ensuring that each gradient descent step obeys differential privacy, the final output model achieves a level of DP through composability. However, training DNNs using SGD involves many iterative processes, so differentially private training algorithms need to monitor the accumulated privacy loss during the process and terminate when necessary to avoid exceeding the privacy budget. 
Research by Abadi \emph{et al.}~\cite{abadi2016deep} has shown that existing strong composition theorems~\cite{dwork2010boosting} lack a precise analysis of DP. To overcome this limitation, MA was introduced. This method tracks the logarithmic moments of privacy loss variables and provides more accurate estimates of privacy loss when combining Gaussian mechanisms under random sampling. Additionally, Mironov \emph{et al.}~\cite{mironov2019r} propose a novel analysis method based on RDP to compute privacy accuracy. This method outperforms the MA method and significantly improves the performance of DPSGD.
Further research~\cite{chen2020stochastic, nasr2020improving, papernot2021tempered, tramer2020differentially} explores methods that involve slight modifications to the model structure or learning algorithm. For example, advanced approaches replace traditional activation functions (such as ReLU) in convolutional neural networks (CNNs) with more gentle Sigmoid functions. Alternatively, research~\cite{pichapati2019adaclip, xiang2019differentially, yu2019differentially, fu2024dpsur} focuses on optimizations such as adaptive adjustment of gradient clipping bounds or dynamically partitioning the privacy budget. However, due to the excessive noise introduced by improper privacy budget partitioning in each epoch, DPSGD still struggles to achieve a satisfactory balance between privacy and utility.

% \subsection{Differentially Private Graph Embedding Generation}
\noindent\textbf{Graph Embedding Generation with DP.}
There are two main types of techniques for generating graph embeddings: neural network-based methods and matrix factorization-based methods.
In this paper, we focus on the former category due to the scalability limitations of matrix factorization-based methods when dealing with large graphs.
The introduction of word2vec~\cite{mikolov2013efficient, mikolov2013distributed} has had a significant impact on neural network-based methods, leading to the development of various techniques such as DeepWalk~\cite{perozzi2014deepwalk}, LINE~\cite{tang2015line}, PTE~\cite{tang2015pte}, and node2vec~\cite{grover2016node2vec}.
Building on this foundation, 
Ahuja \emph{et al.}~\cite{ahuja2020differentially} propose a private learning technique for sparse location data that combines skip-gram with differentially private stochastic gradient descent. However, their approach suffers from poor utility due to excessive partitioning of the privacy budget.
Peng \emph{et al.}~\cite{peng2021differentially} introduce a decentralized scalable learning framework that enables privacy-preserving learning of embeddings from multiple knowledge graphs in an asynchronous and peer-to-peer manner.
Han \emph{et al.}~\cite{han2022framework} develop a general framework for adapting knowledge graph embedding generation algorithms into differentially private versions.
Pan \emph{et al.}~\cite{pan2022fedwalk} present a federated unsupervised graph embedding generation that robustly captures graph structures, ensures DP, and operates with high communication efficiency. However, these studies, including the perturbation mechanisms, face challenges in maintaining utility similar to Ahuja \emph{et al.}~\cite{ahuja2020differentially}. Yang \emph{et al.}~\cite{yang2020secure} propose differentially private GAN and differentially private VAE models for graph synthesis and link prediction. However, despite the low sensitivity of this approach, it tends to converge prematurely under MA, particularly with a limited privacy budget, resulting in diminished performance in both privacy and utility.
Another area of research~\cite{olatunji2021releasing, daigavane2021node, zhang2022towards, sajadmanesh2023gap, sajadmanesh2023progap, xiang2023preserving} within the field of differential private graph learning focuses on GNNs.
The aggregation perturbation (AP) technique is commonly employed as a leading method for ensuring DP in GNNs. Unlike traditional DPSGD algorithms, most existing differentially private GNN methods perturb aggregate information obtained from the GNN neighborhood aggregation step. However, AP encounters compatibility issues with standard GNN architectures, necessitating the re-perturbation of all aggregate outputs at each training iteration, resulting in high privacy costs. 
Furthermore, none of all the existing methods can offer structure-preference settings. Setting preferences is essential for extracting specific graph structures that align with mining objectives, improve predictive accuracy, and yield meaningful insights.

\section{Conclusion}\label{ConcluSection}
In this paper, we have presented SE-PrivGEmb, a structure-preference enabled graph embedding generation under differential privacy using skip-gram. There are two technical highlights. 
First, we deeply analyze the optimization of skip-gram and design a noise tolerance mechanism via non-zero vector perturbation.
This mechanism accommodates arbitrary structure preferences and mitigates utility degradation caused by high sensitivity.
Second, by carefully designing negative sampling probabilities in skip-gram, we prove theoretically that skip-gram can preserve arbitrary node proximities, thereby aligning with mining objectives, improve predictive accuracy, and yield meaningful insights.
Through privacy analysis, we prove that the published low-dimensional node vectors satisfy node-level RDP. Extensive experiments on six real-world graph datasets show that our solution substantially outperforms state-of-the-art competitors. 
For future work, we plan to extend our method to attribute graphs that include billions of nodes or edges. Notably, the attributes associated with the nodes are independent and can be easily managed due to their low sensitivity.
Additionally, inspired by the dynamic graph embedding discussed in~\cite{du2018dynamic}, we also plan to extend our method to dynamic graph embedding while obeying differential privacy. Addressing dynamic graphs will face two significant challenges: allocating privacy budgets to each data element at each version and managing noise accumulation during continuous data publishing. 

% Acknowledgments Section
\section*{Acknowledgments}
This work was supported by the National Natural Science Foundation of China (Grant No: 92270123, 62072390, and 62372122), and the Research Grants Council, Hong Kong SAR, China (Grant No: 15203120, 15226221, 15209922, 15208923 and 15210023).

\bibliographystyle{IEEEtran}
\bibliography{mybibfile}

% Generated by IEEEtran.bst, version: 1.14 (2015/08/26)
\begin{thebibliography}{10}
\providecommand{\url}[1]{#1}
\csname url@samestyle\endcsname
\providecommand{\newblock}{\relax}
\providecommand{\bibinfo}[2]{#2}
\providecommand{\BIBentrySTDinterwordspacing}{\spaceskip=0pt\relax}
\providecommand{\BIBentryALTinterwordstretchfactor}{4}
\providecommand{\BIBentryALTinterwordspacing}{\spaceskip=\fontdimen2\font plus
\BIBentryALTinterwordstretchfactor\fontdimen3\font minus
  \fontdimen4\font\relax}
\providecommand{\BIBforeignlanguage}[2]{{%
\expandafter\ifx\csname l@#1\endcsname\relax
\typeout{** WARNING: IEEEtran.bst: No hyphenation pattern has been}%
\typeout{** loaded for the language `#1'. Using the pattern for}%
\typeout{** the default language instead.}%
\else
\language=\csname l@#1\endcsname
\fi
#2}}
\providecommand{\BIBdecl}{\relax}
\BIBdecl

\bibitem{abadi2016deep}
M.~Abadi, A.~Chu, I.~Goodfellow, H.~B. McMahan, I.~Mironov, K.~Talwar, and
  L.~Zhang, ``Deep learning with differential privacy,'' in \emph{ACM SIGSAC
  Conference on Computer and Communications Security}, 2016, pp. 308--318.

\bibitem{yang2020secure}
C.~Yang, H.~Wang, K.~Zhang, L.~Chen, and L.~Sun, ``Secure deep graph generation
  with link differential privacy,'' in \emph{International Joint Conference on
  Artificial Intelligence}, 2021, pp. 3271--3278.

\bibitem{olatunji2021releasing}
I.~E. Olatunji, T.~Funke, and M.~Khosla, ``Releasing graph neural networks with
  differential privacy guarantees,'' \emph{arXiv preprint arXiv:2109.08907},
  2021.

\bibitem{daigavane2021node}
A.~Daigavane, G.~Madan, A.~Sinha, A.~G. Thakurta, G.~Aggarwal, and P.~Jain,
  ``Node-level differentially private graph neural networks,'' \emph{arXiv
  preprint arXiv:2111.15521}, 2021.

\bibitem{zhang2022towards}
Q.~Zhang, H.~k. Lee, J.~Ma, J.~Lou, C.~Yang, and L.~Xiong, ``{DPAR}: Decoupled
  graph neural networks with node-level differential privacy,'' in
  \emph{Proceedings of the ACM on Web Conference}, 2024, pp. 1170--1181.

\bibitem{sajadmanesh2023gap}
S.~Sajadmanesh, A.~S. Shamsabadi, A.~Bellet, and D.~Gatica-Perez, ``{GAP}:
  Differentially private graph neural networks with aggregation perturbation,''
  in \emph{USENIX Security Symposium}, 2023, pp. 3223--3240.

\bibitem{sajadmanesh2023progap}
S.~Sajadmanesh and D.~Gatica-Perez, ``{P}ro{GAP}: Progressive graph neural
  networks with differential privacy guarantees,'' in \emph{ACM International
  Conference on Web Search and Data Mining}, 2024, pp. 596--605.

\bibitem{xiang2023preserving}
Z.~Xiang, T.~Wang, and D.~Wang, ``Preserving node-level privacy in graph neural
  networks,'' in \emph{IEEE Symposium on Security and Privacy}, 2024, pp.
  4714--4732.

\bibitem{perozzi2014deepwalk}
B.~Perozzi, R.~Al-Rfou, and S.~Skiena, ``{D}eep{W}alk: Online learning of
  social representations,'' in \emph{ACM SIGKDD Conference on Knowledge
  Discovery and Data Mining}, 2014, pp. 701--710.

\bibitem{tang2015line}
J.~Tang, M.~Qu, M.~Wang, M.~Zhang, J.~Yan, and Q.~Mei, ``{LINE}: Large-scale
  information network embedding,'' in \emph{ACM International Conference on
  World Wide Web}, 2015, pp. 1067--1077.

\bibitem{tang2015pte}
J.~Tang, M.~Qu, and Q.~Mei, ``{PTE}: Predictive text embedding through
  large-scale heterogeneous text networks,'' in \emph{ACM SIGKDD Conference on
  Knowledge Discovery and Data Mining}, 2015, pp. 1165--1174.

\bibitem{grover2016node2vec}
A.~Grover and J.~Leskovec, ``node2vec: Scalable feature learning for
  networks,'' in \emph{ACM SIGKDD Conference on Knowledge Discovery and Data
  Mining}, 2016, pp. 855--864.

\bibitem{du2018dynamic}
L.~Du, Y.~Wang, G.~Song, Z.~Lu, and J.~Wang, ``Dynamic network embedding: An
  extended approach for skip-gram based network embedding,'' in
  \emph{International Joint Conference on Artificial Intelligence}, 2018, pp.
  2086--2092.

\bibitem{dwork2006calibrating}
C.~Dwork, F.~McSherry, K.~Nissim, and A.~Smith, ``Calibrating noise to
  sensitivity in private data analysis,'' in \emph{Theory of Cryptography
  Conference}, 2006, pp. 265--284.

\bibitem{hay2009accurate}
M.~Hay, C.~Li, G.~Miklau, and D.~Jensen, ``Accurate estimation of the degree
  distribution of private networks,'' in \emph{IEEE International Conference on
  Data Mining}, 2009, pp. 169--178.

\bibitem{dwork2006differential}
C.~Dwork, ``Differential privacy,'' in \emph{International Colloquium on
  Automata, Languages, and Programming}, 2006, pp. 1--12.

\bibitem{mironov2017renyi}
I.~Mironov, ``R{\'e}nyi differential privacy,'' in \emph{IEEE Computer Security
  Foundations Symposium}, 2017, pp. 263--275.

\bibitem{barabasi1999emergence}
A.-L. Barab{\'a}si and R.~Albert, ``Emergence of scaling in random networks,''
  \emph{Science}, vol. 286, no. 5439, pp. 509--512, 1999.

\bibitem{zhou2009predicting}
T.~Zhou, L.~L{\"u}, and Y.-C. Zhang, ``Predicting missing links via local
  information,'' \emph{The European Physical Journal B}, vol.~71, pp. 623--630,
  2009.

\bibitem{katz1953new}
L.~Katz, ``A new status index derived from sociometric analysis,''
  \emph{Psychometrika}, vol.~18, no.~1, pp. 39--43, 1953.

\bibitem{haveliwala2002topic}
T.~H. Haveliwala, ``Topic-sensitive pagerank,'' in \emph{ACM International
  Conference on World Wide Web}, 2002, pp. 517--526.

\bibitem{yang2015network}
C.~Yang, Z.~Liu, D.~Zhao, M.~Sun, and E.~Chang, ``Network representation
  learning with rich text information,'' in \emph{International Joint
  Conference on Artificial Intelligence}, 2015, pp. 2111--2117.

\bibitem{he2019model}
Z.~He, T.~Zhang, and R.~B. Lee, ``Model inversion attacks against collaborative
  inference,'' in \emph{Annual Computer Security Applications Conference},
  2019, pp. 148--162.

\bibitem{qiu2018network}
J.~Qiu, Y.~Dong, H.~Ma, J.~Li, K.~Wang, and J.~Tang, ``Network embedding as
  matrix factorization: Unifying {D}eep{W}alk, {LINE}, {PTE}, and node2vec,''
  in \emph{ACM International Conference on Web Search and Data Mining}, 2018,
  pp. 459--467.

\bibitem{wang2019subsampled}
Y.~Wang, B.~Balle, and S.~P. Kasiviswanathan, ``Subsampled r{\'e}nyi
  differential privacy and analytical moments accountant,'' in
  \emph{International Conference on Artificial Intelligence and Statistics},
  2019, pp. 1226--1235.

\bibitem{zhu2019poission}
Y.~Zhu and Y.~Wang, ``Poission subsampled r{\'e}nyi differential privacy,'' in
  \emph{International Conference on Machine Learning}, 2019, pp. 7634--7642.

\bibitem{mironov2019r}
I.~Mironov, K.~Talwar, and L.~Zhang, ``R{\'e}nyi differential privacy of the
  sampled gaussian mechanism,'' \emph{arXiv preprint arXiv:1908.10530}, 2019.

\bibitem{bu2023differentially}
Z.~Bu, Y.~Wang, S.~Zha, and G.~Karypis, ``Differentially private optimization
  on large model at small cost,'' in \emph{International Conference on Machine
  Learning}, 2023, pp. 3192--3218.

\bibitem{jin2021toward}
J.~Jin, M.~Heimann, D.~Jin, and D.~Koutra, ``Toward understanding and
  evaluating structural node embeddings,'' \emph{ACM Transactions on Knowledge
  Discovery from Data}, vol.~16, no.~3, pp. 1--32, 2021.

\bibitem{stark2006biogrid}
C.~Stark, B.-J. Breitkreutz, T.~Reguly, L.~Boucher, A.~Breitkreutz, and
  M.~Tyers, ``Bio{GRID}: A general repository for interaction datasets,''
  \emph{Nucleic Acids Research}, vol.~34, pp. D535--D539, 2006.

\bibitem{zhang2018link}
M.~Zhang and Y.~Chen, ``Link prediction based on graph neural networks,'' in
  \emph{International Conference Neural Information Processing Systems}, 2018,
  pp. 5165--5175.

\bibitem{du2022understanding}
L.~Du, X.~Chen, F.~Gao, Q.~Fu, K.~Xie, S.~Han, and D.~Zhang, ``Understanding
  and improvement of adversarial training for network embedding from an
  optimization perspective,'' in \emph{ACM International Conference on Web
  Search and Data Mining}, 2022, pp. 230--240.

\bibitem{lai2017prune}
Y.~Lai, C.~Hsu, W.~H. Chen, M.~Yeh, and S.~Lin, ``{PRUNE}: Preserving proximity
  and global ranking for network embedding,'' in \emph{International Conference
  Neural Information Processing Systems}, 2017, pp. 5257--5266.

\bibitem{tu2018unified}
C.~Tu, X.~Zeng, H.~Wang, Z.~Zhang, Z.~Liu, M.~Sun, B.~Zhang, and L.~Lin, ``A
  unified framework for community detection and network representation
  learning,'' \emph{IEEE Transactions on Knowledge and Data Engineering},
  vol.~31, no.~6, pp. 1051--1065, 2018.

\bibitem{sweeney2002k}
L.~Sweeney, ``\emph{k}-anonymity: A model for protecting privacy,''
  \emph{International Journal of Uncertainty, Fuzziness and Knowledge-Based
  Systems}, vol.~10, no.~05, pp. 557--570, 2002.

\bibitem{machanavajjhala2007diversity}
A.~Machanavajjhala, D.~Kifer, J.~Gehrke, and M.~Venkitasubramaniam,
  ``\emph{l}-diversity: Privacy beyond \emph{k}-anonymity,'' \emph{ACM
  Transactions on Knowledge Discovery from Data}, vol.~1, no.~1, pp. 3--es,
  2007.

\bibitem{hu2014private}
H.~Hu, J.~Xu, X.~Xu, K.~Pei, B.~Choi, and S.~Zhou, ``Private search on
  key-value stores with hierarchical indexes,'' in \emph{IEEE International
  Conference on Data Engineering}, 2014, pp. 628--639.

\bibitem{chaudhuri2011differentially}
K.~Chaudhuri, C.~Monteleoni, and A.~D. Sarwate, ``Differentially private
  empirical risk minimization,'' \emph{Journal of Machine Learning Research},
  vol.~12, no.~3, 2011.

\bibitem{bassily2014private}
R.~Bassily, A.~Smith, and A.~Thakurta, ``Private empirical risk minimization:
  Efficient algorithms and tight error bounds,'' in \emph{IEEE Annual Symposium
  on Foundations of Computer Science}, 2014, pp. 464--473.

\bibitem{zhang2018calm}
Z.~Zhang, T.~Wang, N.~Li, S.~He, and J.~Chen, ``{CALM}: Consistent adaptive
  local marginal for marginal release under local differential privacy,'' in
  \emph{ACM SIGSAC Conference on Computer and Communications Security}, 2018,
  pp. 212--229.

\bibitem{ye2021beyond}
Q.~Ye, H.~Hu, N.~Li, X.~Meng, H.~Zheng, and H.~Yan, ``{B}eyond {V}alue
  {P}erturbation: Local differential privacy in the temporal setting,'' in
  \emph{IEEE INFOCOM Conference on Computer Communications}, 2021, pp. 1--10.

\bibitem{zhang2021privsyn}
Z.~Zhang, T.~Wang, N.~Li, J.~Honorio, M.~Backes, S.~He, J.~Chen, and Y.~Zhang,
  ``{P}riv{S}yn: Differentially private data synthesis,'' in \emph{USENIX
  Security Symposium}, 2021, pp. 929--946.

\bibitem{ye2023stateful}
Q.~Ye, H.~Hu, K.~Huang, M.~H. Au, and Q.~Xue, ``{S}tateful {S}witch: Optimized
  time series release with local differential privacy,'' in \emph{IEEE INFOCOM
  Conference on Computer Communications}, 2023, pp. 1--10.

\bibitem{Zhang2023trajectory}
Y.~Zhang, Q.~Ye, R.~Chen, H.~Hu, and Q.~Han, ``Trajectory data collection with
  local differential privacy,'' in \emph{Proceedings of the VLDB Endowment},
  2023, pp. 2591--2604.

\bibitem{mcmahan2017learning}
H.~B. McMahan, D.~Ramage, K.~Talwar, and L.~Zhang, ``Learning differentially
  private language models without losing accuracy,'' \emph{arXiv preprint
  arXiv:1710.06963}, 2017.

\bibitem{song2013stochastic}
S.~Song, K.~Chaudhuri, and A.~D. Sarwate, ``Stochastic gradient descent with
  differentially private updates,'' in \emph{IEEE Global Conference on Signal
  and Information Processing}, 2013, pp. 245--248.

\bibitem{dwork2010boosting}
C.~Dwork, G.~N. Rothblum, and S.~Vadhan, ``Boosting and differential privacy,''
  in \emph{IEEE Annual Symposium on Foundations of Computer Science}, 2010, pp.
  51--60.

\bibitem{chen2020stochastic}
C.~Chen and J.~Lee, ``Stochastic adaptive line search for differentially
  private optimization,'' in \emph{IEEE International Conference on Big Data},
  2020, pp. 1011--1020.

\bibitem{nasr2020improving}
M.~Nasr, R.~Shokri, and A.~Houmansadr, ``Improving deep learning with
  differential privacy using gradient encoding and denoising,'' \emph{arXiv
  preprint arXiv:2007.11524}, 2020.

\bibitem{papernot2021tempered}
N.~Papernot, A.~Thakurta, S.~Song, S.~Chien, and {\'U}.~Erlingsson, ``Tempered
  sigmoid activations for deep learning with differential privacy,'' in
  \emph{AAAI Conference on Artificial Intelligence}, vol.~35, no.~10, 2021, pp.
  9312--9321.

\bibitem{tramer2020differentially}
F.~Tramer and D.~Boneh, ``Differentially private learning needs better features
  (or much more data),'' \emph{arXiv preprint arXiv:2011.11660}, 2020.

\bibitem{pichapati2019adaclip}
V.~Pichapati, A.~T. Suresh, F.~X. Yu, S.~J. Reddi, and S.~Kumar, ``{A}da{CliP}:
  Adaptive clipping for private {SGD},'' \emph{arXiv preprint
  arXiv:1908.07643}, 2019.

\bibitem{xiang2019differentially}
L.~Xiang, J.~Yang, and B.~Li, ``Differentially-private deep learning from an
  optimization perspective,'' in \emph{IEEE INFOCOM Conference on Computer
  Communications}, 2019, pp. 559--567.

\bibitem{yu2019differentially}
L.~Yu, L.~Liu, C.~Pu, M.~E. Gursoy, and S.~Truex, ``Differentially private
  model publishing for deep learning,'' in \emph{IEEE Symposium on Security and
  Privacy}, 2019, pp. 332--349.

\bibitem{fu2024dpsur}
J.~Fu, Q.~Ye, H.~Hu, Z.~Chen, L.~Wang, K.~Wang, and X.~Ran, ``{DPSUR}:
  Accelerating differentially private stochastic gradient descent using
  selective update and release,'' in \emph{Proceedings of the VLDB Endowment},
  2024, pp. 1200--1213.

\bibitem{mikolov2013efficient}
T.~Mikolov, K.~Chen, G.~Corrado, and J.~Dean, ``Efficient estimation of word
  representations in vector space,'' \emph{arXiv preprint arXiv:1301.3781},
  2013.

\bibitem{mikolov2013distributed}
T.~Mikolov, I.~Sutskever, K.~Chen, G.~S. Corrado, and J.~Dean, ``Distributed
  representations of words and phrases and their compositionality,'' in
  \emph{International Conference Neural Information Processing Systems}, 2013,
  pp. 3111--3119.

\bibitem{ahuja2020differentially}
R.~Ahuja, G.~Ghinita, and C.~Shahabi, ``Differentially-private next-location
  prediction with neural networks,'' in \emph{International Conference on
  Extending Database Technology}, 2020, pp. 121--132.

\bibitem{peng2021differentially}
H.~Peng, H.~Li, Y.~Song, V.~Zheng, and J.~Li, ``Differentially private
  federated knowledge graphs embedding,'' in \emph{ACM International Conference
  on Information and Knowledge Management}, 2021, pp. 1416--1425.

\bibitem{han2022framework}
X.~Han, D.~Dell’Aglio, T.~Grubenmann, R.~Cheng, and A.~Bernstein, ``A
  framework for differentially-private knowledge graph embeddings,''
  \emph{Journal of Web Semantics}, vol.~72, p. 100696, 2022.

\bibitem{pan2022fedwalk}
Q.~Pan and Y.~Zhu, ``{F}ed{W}alk: Communication efficient federated
  unsupervised node embedding with differential privacy,'' in \emph{ACM SIGKDD
  Conference on Knowledge Discovery and Data Mining}, 2022, pp. 1317--1326.

\end{thebibliography}
\end{document}